\documentclass[11pt]{article}

\usepackage{microtype}
\usepackage{graphicx}
\usepackage{subfigure}
\usepackage{booktabs} 
\usepackage[shortlabels]{enumitem}
\usepackage{fullpage} 
\usepackage{hyperref}

\usepackage{amsmath}
\usepackage{amssymb}
\usepackage{mathtools}
\usepackage{amsthm}
\usepackage{xspace}
\usepackage[capitalize,noabbrev]{cleveref}

\usepackage[utf8]{inputenc}
\usepackage{afterpage}

\usepackage{epstopdf}
\usepackage{underlin}
\usepackage{fancyhdr}

\usepackage{pdfpages}
\usepackage{pifont}

\usepackage{color}
\usepackage{pgfplots}
\usepackage{cite} 
\usepackage{fullpage}

\theoremstyle{plain}
\newtheorem{theorem}{Theorem}[section]
\newtheorem{example}[theorem]{Example}
\newtheorem{proposition}[theorem]{Proposition}
\newtheorem{lemma}[theorem]{Lemma}
\newtheorem{corollary}[theorem]{Corollary}
\theoremstyle{definition}
\newtheorem{definition}[theorem]{Definition}

\theoremstyle{remark}
\newtheorem{remark}[theorem]{Remark}

\title{What Functions Can Graph Neural Networks Generate?}

\newcommand{\R}{\mathbb{R}}
\newcommand{\C}{\mathbb{C}}
\newcommand{\F}{\mathcal{F}}
\newcommand{\G}{\mathcal{G}}
\newcommand{\GPC}{\text{permutation-compatible\xspace}}

\author{
{ \normalfont Mohammad Fereydounian}\thanks{Department of Electrical and Systems Engineering, University of Pennsylvania, Philadelphia, PA, USA. Emails: \ttfamily\bfseries\{mferey, hassani\}@seas.upenn.edu.} 
\and {Hamed Hassani}\footnotemark[1] 
\and {Amin Karbasi}\thanks{Department of Electrical Engineering, Computer Science, Statsitics \& Data Science, Yale University, New Haven, CT, USA. Email: \ttfamily\bfseries amin.karbasi@yale.edu.} 
}

\date{}
\begin{document}

\maketitle

\begin{abstract}
In this paper, we fully answer the above question through a key algebraic condition on graph functions, called \textit{permutation compatibility}, that relates permutations of weights and features of the graph to functional constraints. We prove that: (i) a GNN, as  a graph function, is necessarily permutation compatible; (ii) conversely, any permutation compatible function, when restricted on input graphs with distinct node features, can be generated by a GNN; (iii) for arbitrary node features (not necessarily distinct), a simple feature augmentation scheme suffices to generate a permutation compatible function by a GNN; (iv) permutation compatibility can be verified by checking only quadratically many functional constraints, rather than an exhaustive search over all the permutations; (v) GNNs can generate \textit{any} graph function once we augment the node features with node identities, thus going beyond graph isomorphism and permutation compatibility. The above characterizations pave the path to formally study the intricate connection between GNNs and other algorithmic procedures on graphs. For instance, our characterization implies that many natural graph problems, such as min-cut value,  max-flow value, max-clique size, and shortest path can be generated by a GNN using a simple feature augmentation. In contrast, the celebrated Weisfeiler-Lehman graph-isomorphism test fails whenever a permutation compatible function with identical features cannot be generated by a GNN.  At the heart of our analysis lies a novel representation theorem that identifies  basis functions for GNNs. This enables us to translate the properties of the target graph function into properties of the GNN's aggregation function. 
\end{abstract}

\section{Introduction}\label{sec: intro}
Processing data with graph structures has become an essential tool in application domains such as as computer vision \cite{yang2019auto}, natural language processing \cite{yu2020reinforce}, recommendation systems \cite{wang2019kgat}, and drug discovery \cite{ma2018constrained}, to name a few. Graph Neural Networks (GNN) are a class of iterative-based models that can process information represented in the form of graphs. Through a message passing mechanism, GNNs aggregate information from neighboring nodes in the graph in order to update node features \cite{gilmer2017neural}. Such node features can be ultimately used for down-stream tasks such as classification, link prediction, clustering,~etc.  

Even though many variations and architectures of GNNs have been proposed in recent years to increase the representation capacity of GNNs \cite{kipf2017semi,hamilton2017inductive,velinckovic2018graph,xu2018representation,santoro2018measuring,gama2018convolutional, verma2018graph,ying2018hierarchical,zhang2018end,ruiz2020graphon,ruiz2019invariance,scarselli2008graph,battaglia2016interaction,velickovic2020Neural,duvenaud2015convolutional,kearnes2016molecular,li2015gated,defferrard2016convolutional,xu2018how}, 
it is still not clear what class of functions GNNs can generate exactly.    There has been a large body of work that aims to understand the expressive power of GNNs through their ability to distinguish non-isomorphic graphs and the Weisfeiler–Lehman  graph isomorphism test \cite{weisfeiler1968reduction, scarselli2009comp, xu2018how, morris2019weisfeiler}. However, the aforementioned results do not provide much indication to practitioners  whether a specific graph function (e.g., shortest paths, min-cut, etc) can be computed by a GNN.  In this paper, we aim to provide an exact  characterization of how a given graph problem can be solved by GNNs. Our results are analogous to those of approximation capabilities of the feedforward neural networks on the space of continuous functions \cite{bartlett2009NN}. 

More specifically, we consider graphs that consist of nodes equipped with feature vectors, along with weights  assigned to \textit{all} pairs of nodes (i.e., edges). We should note  that almost all graph problems can be stated over fully  connected but weighted graphs. For example, for computing the shortest path on a given graph (which may not be fully connected), we can assign a very large value to non-existing edges.  A graph function takes as input a graph in the form of weight and feature matrices and assigns a vector to each node. Similarly, a GNN is an evolving graph function that updates node features iteratively through an aggregation operation. Naturally, for GNNs to be able to solve graph problems defined over weighted graphs, their message-passing iterates need to incorporate edge weights.  Finally, in our setting, we do not generally consider  pooling/readout operations, since such operations can considerably reduce the class of functions generated by a GNN. However, as we  will discuss shortly in related work, our results have important implications on GNNs with readouts.

\textbf{Our Contributions} are summarized as follows: 
\begin{enumerate}
	
	\item We provide an algebraic condition, so called \textit{permutation-compatibility} that
	relates permutations of weights and features of the graph to functional constraints. This condition will be used as a key notion in  characterizing the representation power of GNNs. Indeed, we show that a GNN, as  a graph function, is necessarily permutation compatible. 
	\item Conversely, any permutation-compatible function, when restricted on input graphs with distinct node features, can be generated by a GNN. Further,   for arbitrary node features (not necessarily distinct), a simple feature augmentation scheme suffices to generate a permutation-compatible function by a GNN. 
	\item We show that for any graph problem, permutation compatibility can be verified  over  quadratically many constraints rather than an exhaustive search over exponentially many permutations. 
	\item We characterize the basis functions  for \GPC~graph functions. These basis functions effectively relate the properties of aggregation operators to the expressive power of the resulting GNNs. For instance, it follows that with continuous aggregation operators, all continuous \GPC~functions lie within the reach of GNNs. 
	\item Going beyond permutation compatibility and graph isomorphism, we show that  GNNs can generate \textit{any} graph function once we augment the node features with node identities. Such feature augmentations then allow us to study the connection between GNNs and other iterative graph procedures such as dynamic programs. 
\end{enumerate}

\subsection{Related Work}
It is well-established that GNNs  cannot assign different values to isomorphic graphs \cite{scarselli2009comp}. Moreover, from \cite{xu2018how} and \cite{morris2019weisfeiler} we know that GNNs with appropriate \emph{aggregation} and \emph{pooling} operators, over  \textit{unweighted} graphs,  are only as powerful as the color refinement of the Weisfeiler–Lehman  graph isomorphism test, denoted by 1-WL \cite{weisfeiler1968reduction}. Due to this negative result, many follow-up works proposed more involved variants  such as  as adding stochastic features \cite{murphy2018janossy,sato2021random,you2019position,Srinivasan2020On,dwivedi2020benchmarking}, adding deterministic distance features \cite{li2020distance,you2021identity}, or building higher order GNNs \cite{morris2019weisfeiler,maron2018invariant,maron2019universality,maron2019provably,chen2019equivalence}, so that the expressive power of the resulting GNNs go beyond the 1-WL test \cite{geerts2022expressiveness}. In this light, we establish in Section~\ref{sec: wl-connection} a precise connection between permutation compatibility and 1-WL test on unweighted graphs. Note that in our GNN setting, we consider fully connected weighted graphs without the pooling/readout operation. As a result, the equivalence between GNNs (on unweighted graphs with readout mechanisms) and 1-WL test do not directly apply to our setting. Indeed, our precise characterization of graph functions generated by GNNs, namely permutation-compatibility, also allows us to shed light on some of the elusive features of GNNs.

\textbf{Implications of our results.} One of the main theoretical directions with regard to the expressive power of GNNs has been through establishing an alignment between the iterative updates of a GNN and the 1-WL test \cite{xu2018how,morris2019weisfeiler}. However, our results are of a different nature. Given any graph function, our representation theorem provides explicit choices for a GNN that generates the function (possibly with appropriate feature augmentation). In this sense, our results are in nature similar to the ones showing that neural networks are universal function approximators \cite{bartlett2009NN}, or the ones showing that deep sets  can approximate any permutation-invariant function \cite{zaheer2017deepset}. A similar comparison can be made between our results and the recent works on the alignment of GNNs with the dynamic programming approaches for specific graph problems such as the shortest path problem \cite{xu2020what,dudzik2022graph}. Indeed, our results prove (via construction) the existence of GNNs  that can solve a graph problem (such as shortest path, min-cut, max-flow, etc) once the features are properly augmented.

\section{Preliminaries}\label{sec: pre}

Throughout the paper, we consider multi-dimensional arrays (sequences) of objects. By $(a_1,\ldots,a_n)$, we denote a one-dimensional array of objects $a_1,\ldots,a_n$. If $A=(a_1,\ldots,a_n)$, we refer to the $i$-th element of $A$ by $[A]_i$, i.e., $[A]_i=a_i$. Similarly, if $A$ is a two-dimensional array (e.g., a matrix), $[A]_{i,j}$ refers to its $(i,j)$-th element.
Similarly, $[A]_{i,j,k}$ refers to the $(i,j,k)$-th element in a three-dimensional array $A$. Sets are denoted by $\{\cdot\}$.
We also let $[n]=\{1,\ldots,n\}$ and $[n]_{-i}= \{1,\ldots,n\}\setminus \{i\}$, where $A\setminus B$ denotes the set difference. The set of complex numbers is denoted by $\mathbb{C}$. If $z\in \mathbb{C}$, we use the standard notation  $z=\operatorname{Re}(z)+\operatorname{Im}(z)\sqrt{-1}$.

In the following, we formally define {\it graphs}, {\it graph functions}, and  GNNs.

\begin{definition}[Class $\G_{n,d}$ graphs]\label{def: graph}
	An undirected graph $G$ is a tuple $G=([n],W,X)$, where $[n]=\{1,\ldots,n\}$ is the set of nodes, and every  pair of nodes $\{i,j\}$ with $i\neq j$ forms an edge to which a weight $w_{i,j}=w_{j,i}$ is assigned.  The symmetric matrix $W\in \R^{n\times n}$ is called the weight matrix with zeros on its diagonal. Further, each node $i$ is associated with a row feature vector $x_i\in \R^d$. We call  $X=(x_1^{\top},\ldots,x_n^{\top})$  the feature matrix. Finally, we  denote by  $\mathcal{G}_{n,d}$ the set of  graphs of size  $n$ with feature vectors of dimension $d$.
\end{definition}
\begin{remark}
	For the ease of presentation, we mainly consider scalar-valued weights. However,   all of results can be extended to vector-valued weights.
\end{remark}

\begin{definition}[Graph function]\label{def: func-graph}
	A graph function over $\G_{n,d}$  is a function $F$ that takes as input any graph $G=([n],W,X)\in\G_{n,d}$ and is identified  by its action  on $(W,X)$ via the following form: $F(W,X) = (f_1(W,X),\ldots,f_n(W,X))$, where $f_i(W,X)$, so called the node-functions, are vector-valued functions in some common Euclidean vector space.
\end{definition}
\begin{example}\label{ex: funcs}
	To better understand the notion of graph functions, let us consider a few examples.  
	\begin{enumerate}
		
		\item  \label{3-node}{\bf Feature-Oblivious.} Let $n=3$, and consider a function $F$ with $f_1(W,X) = 0$, $f_2(W,X) = w_{1,2}+w_{2,3}$, and $f_3(W,X) = \sin\left(w_{1,3}+w_{2,3}\right)$.
		
		
		\item  \label{feature-sum}{\bf Feature-Sum.}  Let $f_i(W,X)=\sum_{j \in [n]} x_{j}$.
		
		\item  \label{minx_i} {\bf Min-Sum.} Let $f_i(W,X)=\min(x_i,\, \sum_{j\in [n]_{-i}} x_j)$ for scalar-valued features, i.e., $d=1$.
		
		\item \label{deg}{\bf Degree.} Let $f_i(W,X) = \sum_{j\in[n]}w_{i,j}$. 
		
		\item \label{max-degree}{\bf Max-Neighbor-Degree.} Let $F$ be a function that assigns to each node $i$ the maximum degree of its neighbors, i.e., $f_i(W,X) = \max_{j\in [n]_{-i}}(\sum_{r\in[n]}w_{r,j})$. 
		
		\item \label{node1}{\bf Distance-to-Node-$1$.} Let $F$ be a function that assigns to each node $i$  the length of its shortest path to node $1$. More formally, $f_1(W,X) = 0$, and for $i\in [n]_{-1}$
		\begin{align}\label{Dist}
			f_i(W,X) = \min_{(j_0,\ldots,j_{\ell})\in P(i,1)}\left(\sum_{r=0}^{\ell-1}w_{j_r,j_{r+1}}\right),
		\end{align}
		where $P(i,1)$ denotes the set of all paths starting from node $i$ and ending in node $1$.
		
		\item \label{cut}{\bf Min-Cut Value.} Let $F$ be a function that assigns to every node the minimum-cut of the whole graph; i.e., for all $i\in [n]$
		\begin{align}\label{min cut}
			f_i(W,X) = \min_{\emptyset\subsetneq A\subsetneq [n]}\left(\sum_{r\in A}\sum_{s\in [n]\setminus A}w_{r,s}\right).
		\end{align}
	\end{enumerate}
\end{example}
\begin{definition}[GNN]\label{def: gnn}
	A Graph Neural Network (GNN) is an iterative mechanism that generates a sequence of functions $H^{(k)}$, for $k\geq 0$, over $\G_{n,d}$ in the following manner.  For $G=([n],W,X)\in\G_{n,d}$, the function $H^{(k)}(W,X)=(h_1^{(k)}(W,X),\ldots,h_n^{(k)}(W,X))$ is given as 
	\begin{align}
		&\text{if }k=0:\, h_i^{(0)} = x_i,\\
		\label{gnn-formula}&\text{if }k\geq 1:\, h_i^{(k)} = \sum_{j\in [n]_{-i}} \phi_{k}\left(h_i^{(k-1)},h_j^{(k-1)},w_{i,j}\right).
	\end{align}
	We assume that  the outputs of functions $\phi_k$, for $k\geq 1$, lie in some Euclidean vector space. 
\end{definition}


Definition~\ref{def: gnn} puts no restriction on the function-class of $\phi_k$. However,   $\phi_k$ is often  chosen from the class of multi-layer perceptrons (MLPs). The update \eqref{gnn-formula} is called the aggregation operator. It is common  in the literature to consider more general aggregation operators. Nevertheless, the following proposition states that these general aggregators
do not enlarge the function-class of GNNs.  For a more formal statement and proof, we refer to \cref{sec: replacing}.
\begin{proposition}[Informal]\label{prop: informal}
	Suppose we replace \eqref{gnn-formula} with an aggregation operation of the form 
	\begin{align}\label{AGG}
		h_i^{(k)}=\operatorname{AGG}\left(h_i^{(k-1)},\left\{(h_j^{(k-1)},w_{i,j})\mid j\in[n]_{-i}\right\}\right).
	\end{align}
	Then the class of functions generated by GNNs under such an aggregation is not larger than the class of functions generated by a GNN with the aggregation defined in \eqref{gnn-formula}.
\end{proposition}

The proof of \cref{prop: informal} shows that in fact every iteration of the form \eqref{AGG} can be represented by two consecutive iterations of the form \eqref{gnn-formula}. Finally, our characterization of the class of functions generated by GNNs requires formalizing the concepts and  notation related to permutations.

\begin{definition}[Permutations]\label{def: perm} A permutation $\pi$ over $[n]$ is a bijective mapping $\pi: [n] \to [n]$. The set of all permutations over $[n]$ is denoted by $S_n$. We also need the following restricted permutations.
	\begin{enumerate}[(i)]
		\item For $i\in[n]$, we use $\nabla_i$ to denote the set of all permutations $\pi$ over $[n]$ such that $\pi(i)=i$. More formally, $\nabla_i =\{\pi\in S_n \mid \pi(i)=i\}$. Here, for simplicity, we have dropped the dependency of $\nabla_i$ on $n$. 
		\item For $i,j\in [n]$, we use $\pi_{i,j}$ to denote the specific permutation over $[n]$ that swaps $i$ and $j$ but fixes all the other elements. More formally, $\pi_{i,j}\in S_n$ with $\pi_{i,j}(i)=j$, $\pi_{i,j}(j)=i$, and $\pi_{i,j}(\ell)=\ell$ for $\ell \in [n]\setminus \{i,j\}$.
	\end{enumerate}
\end{definition}
We next define  permutations on weights and features that are induced by a permutation on the nodes. 
\begin{definition}[Induced Weight-Feature Permutation (IWFP)]\label{def: auto}
	Consider a graph $G=([n],W,X)$ and a permutation $\pi\in S_n$ over its nodes. Then $\pi$ induces a permutation $\sigma_{\pi}$ over the elements of $W$, and a permutation $\lambda_{\pi}$ over the elements of $X$ as follows:  $\sigma_{\pi}(W)$ is an $n\times n$ matrix whose $(i,j)$-th element is $w_{\pi(i),\pi(j)}$. More formally, $[\sigma_{\pi}(W)]_{i,j}=w_{\pi(i),\pi(j)}$. Also, given the feature matrix $X=(x_1^{\top},\ldots,x_n^{\top})$, we have $\lambda_{\pi}(X) = (x_{\pi(1)}^{\top},\ldots,x_{\pi(n)}^{\top})$, or equivalently $[\lambda_{\pi}(X)]_{i} = x_{\pi(i)}^{\top}$. For every $\pi\in S_n$, we call $(\sigma_{\pi},\lambda_{\pi})$, a weight-feature permutation induced by $\pi$.
\end{definition}

\section{Main Results}\label{sec: mr}

In this section, we aim to understand how a graph function can be generated by GNNs. We proceed by introducing the main algebraic structure which our results are based on. 

\begin{definition}[Class of {\GPC} functions $\mathcal{F}_{n,d}$]\label{def: auto-comp}
	Consider a function $F= (f_1,\ldots,f_n)$ over $\G_{n,d}$. We say that  $F$ belongs to the class of {\GPC} functions $\mathcal{F}_{n,d}$ if and only if for every $\pi\in S_n$ and every $G=([n],W,X)\in \G_{n,d}$, we have
	\begin{align}\label{a23}
		f_{\pi{(i)}} \left(W,X \right) = f_i\left(\sigma_{\pi}(W),\lambda_{\pi}(X) \right)\quad \forall i \in [n].
	\end{align}
\end{definition}
Permutation compatibility can be seen as a natural generalization of the permutation-invariance condition (see e.g.  \cite{xu2018how,morris2019weisfeiler}) for graph functions that assigns a node function $f_i$ to each node $i$. We refer to \cref{fig:examples} - (a) for an illustration and \cref{sec: id} for the corresponding examples. 
\begin{figure*}[t]
	\centering
	\subfigure[Illustration of Condition \eqref{a23}.]{\includegraphics[width=0.6\linewidth]{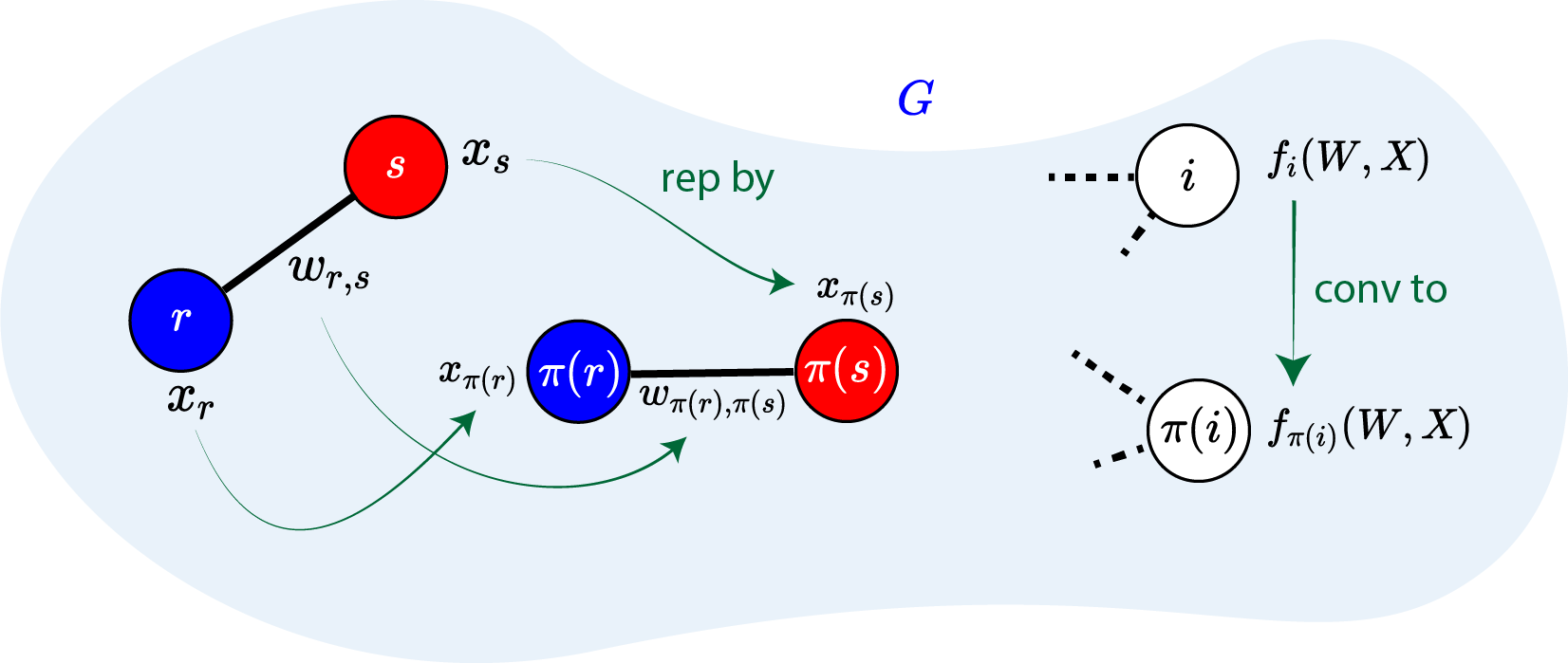}} \hspace{0.5cm}
	\subfigure[\cref{3-node} of \cref{ex: funcs}]{\includegraphics[width=0.35\linewidth]{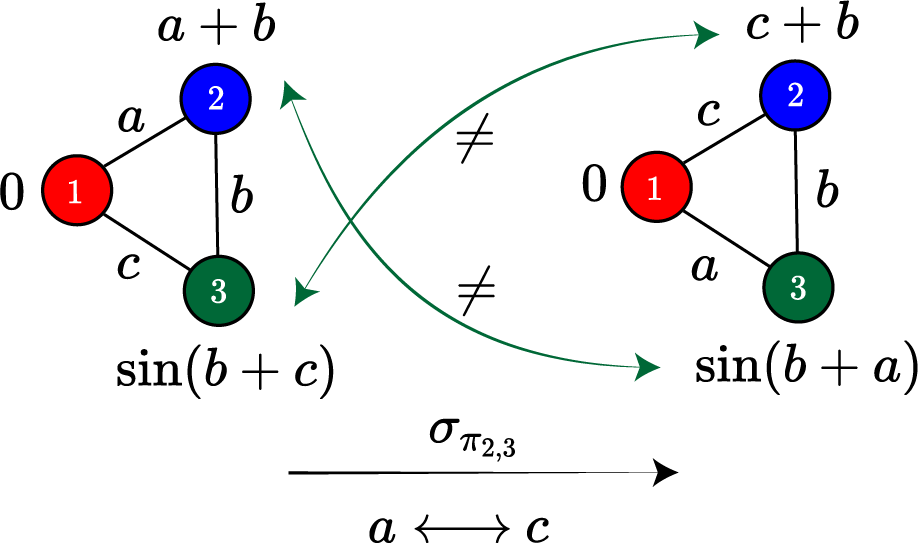}}
	\caption{(a) Illustration of \eqref{a23} which states that replacing all weights $w_{r,s}$ by $w_{\pi(r),\pi(s)}$ and all features $x_r$ by $x_{\pi(r)}$, converts $f_i(W,X)$ into $f_{\pi(i)}(W,X)\,$;  (b) Showing that \cref{3-node} of \cref{ex: funcs} fails to satisfy \eqref{a23} for $\pi=\pi_{2,3}$. For simplicity we used  $w_{1,2}=a$, $w_{2,3}=b$, and $w_{1,3}=c$.}
	\label{fig:examples}
\end{figure*}

\looseness=-1 To demonstrate the connection between permutation compatibility and GNNs, we start by the necessity result, which generalizes the previously-known results on the permutation invariancy of GNNs.
\begin{theorem}[Necessity]\label{theo: necessity}
	Suppose $H^{(k)}(W,X)$ is a GNN over $\mathcal{G}_{n,d}$. For any finite $k\geq 0$, the resulting $H^{(k)}$ is permutation compatible.
\end{theorem}
The above theorem can be formally proven by induction on $k$.  We next proceed with sufficiency results which are far more challenging and, in some cases, require feature augmentation.  Our first sufficiency result  is restricted to graphs with distinct features. 
\begin{definition}\label{def: G-tilde}
	Define $\tilde{\mathcal{G}}_{n,d}$ to be the set of all graphs with distinct node features: 
	\begin{align}
		\tilde{\mathcal{G}}_{n,d} = \left\{G=([n],W,X)\in \mathcal{G}_{n,d} \mid X=(x_1^{\top},\ldots,x_n^{\top}) \text{, where    } x_1,\ldots,x_n \text{  are distinct}\right\}.
	\end{align}
\end{definition}

\begin{theorem}[Sufficiency for distinct features]\label{theo: distinct}
	Suppose $F\in\mathcal{F}_{n,d}$. Then there exists a GNN $H^{(k)}$ with finite $k\geq 0$ such that $H^{(k)}(W,X)=F(W,X)$ for all $G=([n],W,X) \in \tilde{\mathcal{G}}_{n,d}$.
\end{theorem}
\looseness=-1 In the case where the features are identical, we show later in Section~\ref{sec: wl-connection} that  permutation compatibility of a graph function may no longer be sufficient to guarantee that it is generated by a GNN. We establish this result by making a connection with the 1-WL test.  However, as we will show below, a simple augmentation scheme makes it possible to extend \cref{theo: distinct} to any permutation-compatible function. 

\begin{theorem}[Extending the sufficiency to arbitrary features]\label{theo: arbitrary p.c}
	Suppose $F\in\mathcal{F}_{n,d}$. For any graph $G=([n],W,X) \in \mathcal{G}_{n,d}$, let us augment the (row vector) features $x_1,\ldots,x_n$ with arbitrary but distinct vectors, i.e., $x^\ast_i=(x_i,y_i)$, where $y_1,\ldots,y_n\in\R^{d_0}$ are distinct. Let us denote the new feature matrix by $X^\ast=(x^{\ast \top}_1,\ldots,x^{\ast \top}_n)$. Then, there exists a GNN, i.e., $H^{(k)}$ with some finite $k\geq 0$ over $\G_{n,d+d_0}$, such that $H^{(k)}(W,X^{\ast})=F(W,X)$ for all $G=([n],W,X) \in \mathcal{G}_{n,d}$.
\end{theorem}
\cref{theo: distinct} and \cref{theo: arbitrary p.c} are proven in \cref{app: theo-distinct} and \cref{app: theo-arbitrary p.c}, respectively. However, we provide the sketch of the proof in \cref{sec: char}.
\subsection{Verification  of Permutation-Compatible Functions}\label{sec: id}
\looseness=-1 It is easy to see that naively verifying Condition~\eqref{a23} over all permutations leads to $n\cdot n !$ functional constraints. However, it turns out that these constraints can be equivalently represented by a subset of $n(n-1)/2$ constraints.  This is because, 
at a high level, any permutation can be decomposed into a sequence of swaps of a pair of elements (a.k.a transpositions), and thus, invariancy on arbitrary permutations can be verified via the invariancy on transpositions. This  results in the following theorem. 
\begin{proposition}\label{prop: acf-equivalent}
	Consider a function $F = (f_1,\ldots,f_n)$ over $\G_{n,d}$. Then $F\in\mathcal{F}_{n,d}$ if and only if there exists $i_0 \in [n]$ such that both of the following conditions hold for all $G=([n],W,X)\in \G_{n,d}$:
	\begin{align}
		\label{a19}&\text{For all } r,s\in [n]_{-i_0}: \quad f_{i_0}(\sigma_{\pi_{r,s}}(W),\lambda_{\pi_{r,s}}(X))=f_{i_0}(W,X).\\
		\label{a29}&\text{For all } j\in [n]_{-i_0}:\quad f_{j}(W,X) = f_{i_0}(\sigma_{\pi_{i_0,j}}(W),\lambda_{\pi_{i_0,j}}(X)).
	\end{align}
\end{proposition}
In the next two subsections, we consider specific examples of graph functions and determine whether or not they satisfy the conditions stated in \cref{prop: acf-equivalent}.
\subsubsection{Permutation-Compatible Examples} \label{subsec: pos}
Using \cref{prop: acf-equivalent} it is easy to verify that the functions given in \cref{deg} and \cref{max-degree}  of \cref{ex: funcs} are permutation compatible. The following  corollary provides cases for which verifying permutation-compatibility is even simpler than \cref{prop: acf-equivalent}. We refer to \cref{app: pc-ex} for more details. 
\begin{corollary}[Informal]\label{corr-newp} 
	$F = (f_1,\ldots,f_n)$ is permutation compatible if any of the following holds:
	
	(i)  $F$ ignores $W$ and $f_i$ is invariant under any permutation of features $x_j\,$ for  $j\neq i$,
	
	(ii) $F$ assigns the same value to all $f_i$ and this value is invariant under any graph isomorphism.
	
\end{corollary} 
Using \cref{corr-newp} - part (i), we can immediately conclude that Items~ \ref{feature-sum} and \ref{minx_i}  of \cref{ex: funcs} are {\GPC} functions. The implication of part (ii) is expressed in the following remark. 
\begin{remark}\label{rem: min-cut}
	\cref{corr-newp}- part (ii) implies that the min-cut value function given in \cref{cut} of \cref{ex: funcs} is permutation compatible and hence can be generated by a GNN using a distinct feature augmentation due to \cref{theo: arbitrary p.c}. 
	Indeed, many classical graph problems such as the clique number and the max-flow value can be shown to be permutation compatible due to  \cref{corr-newp} - part (ii).
\end{remark}

\subsubsection{Permutation-Incompatible Examples}\label{subsec: neg}
Consider the function $F$ defined in \cref{3-node} of \cref{ex: funcs}. We claim that $F\notin \mathcal{F}_{3,d}$, i.e., $F$ fails to satisfy \eqref{a23}. Let $\pi = \pi_{2,3}$, i.e., we have $\pi(1)=1,\pi(2)=3,\pi(3)=2$. Let $W'=\sigma_{\pi}(W)$ and $X'=\lambda_{\pi}(X)$. Note that $W'$ consists of three elements $w'_{1,2}$, $w'_{1,3}$, and $w'_{2,3}$ and $X'=(x_1',x_2',x_3')$. As shown in \cref{fig:examples}-(b), under the permutation $\sigma_{\pi}$ on the weights, we get $w'_{1,2} = w_{1,3}$, $w'_{1,3} = w_{1,2}$, and  $w'_{2,3} = w_{2,3}$, and under $\lambda_{\pi}$ on the features, we get $x'_1=x_1$, $x'_2=x_3$, and $x'_3=x_2$. We now apply $(\sigma_{\pi},\lambda_{\pi})$ on each node function as follows (note that the specific choice of $F$ considered here totally ignores $X$ and only depends on $W$):
\begin{align}
	f_1(\sigma_{\pi}(W),\lambda_{\pi}(X)) &= f_1(W',X') = 0,\\
	f_2(\sigma_{\pi}(W),\lambda_{\pi}(X)) &= f_2(W',X')=w'_{1,2}+w'_{2,3} =w_{1,3}+w_{2,3},\\
	f_3(\sigma_{\pi}(W),\lambda_{\pi}(X)) &= f_3(W',X') =\sin\left(w'_{1,3}+w'_{2,3}\right)=\sin\left(w_{1,2}+w_{2,3}\right).
\end{align}
\looseness=-1 Since $\pi(2)=3$, guaranteeing \eqref{a23} requires $f_3(W,X) = f_2(\sigma_{\pi}(W),\lambda_{\pi}(X))$, which does not hold as $f_3(W,X)=\sin(w_{1,3}+w_{2,3})$ while $f_2(\sigma_{\pi}(W),\lambda_{\pi}(X))=w_{1,3}+w_{2,3}$. Hence, $F\notin \mathcal{F}_{3,d}$. 
 
\subsection{Permutation Compatibility and Node Labeling}\label{sec: intuition}
\looseness=-1 In this section, we explain that permutation compatibility is a formal way of saying that a function is blind to node identities, i.e., fixing a node, the value that the function assigns to that node remains the same under re-labeling. To see this, pick a permutation $\pi\in S_n$ and re-label the nodes by writing $\pi^{-1}(i)$ instead of $i$. Therefore, the new label for the edge $\{i,j\}$ is now $\{\pi^{-1}(i),\pi^{-1}(j)\}$. Letting $W'=\sigma_{\pi}(W)$ and $X'=\lambda_{\pi}(X)$, note that $w_{\pi^{-1}(i),\,\pi^{-1}(j)}'$ refers the weight of the edge whose new name is $\{\pi^{-1}(i),\pi^{-1}(j)\}$ and $x'_{\pi^{-1}(i)}$ refers to the feature of the node whose name~is $\pi^{-1}(i)$. Indeed,   $w_{\pi^{-1}(i),\,\pi^{-1}(j)}'=w_{\pi(\pi^{-1}(i)),\,\pi(\pi^{-1}(j))}=w_{i,j}$ and $x'_{\pi^{-1}(i)}=x_{\pi(\pi^{-1}(i))}=x_i$. The original function value assigned to node $i$ was $f_i(W,
X)$, and now under the re-labeling the value assigned to the same node (which is now named 
$\pi^{-1}(i)$) is $f_{\pi^{-1}(i)}(W',X')$. If $F$ does not depend on node labelings, these two values should be equal, i.e.,  $f_{\pi^{-1}(i)}(W',X')=f_i(W,X)$ or $f_{\pi^{-1}(i)}(\sigma_{\pi}(W),\lambda_{\pi}(X))=f_i(W,X)$ for all $i\in [n]$. Replacing $i =\pi(k)$ in this equation implies that $f_{k}(\sigma_{\pi}(W),\lambda_{\pi}(X))=f_{\pi(k)}(W,X)$ must hold for all $k$, which is the permutation-compatibility condition in \eqref{a23}.

\subsection{Characterization of $\mathcal{F}_{n,d}$ and Proof Sketch}\label{sec: char}
In this section, we study a characterization of  $\mathcal{F}_{n,d}$ that paves the path for proving \cref{theo: distinct}. Based on the definitions and results of this section, \cref{theo: distinct} is proven in \cref{app: theo-distinct}. To reach the result of \cref{theo: distinct}, we take the following steps:

\noindent{\bf (i) Building MEF functions.} We start by introducing the notion of {\it multiset-equivalent functions} (MEF) and provide useful candidates for such functions. MEFs are building blocks for defining the basis functions in step (ii).
\begin{definition}[Multiset-Equivalent Function (MEF)]\label{def: cef}
	For positive integers $n$ and $m$, we call the function $\psi:\R^m \to \R^{p}$  a multiset-equivalent function (MEF), if for all $v_1,\ldots,v_n, v_1',\ldots,v_n'\in \R^m$, the equation
	\begin{align}\label{a4}
		\sum_{i=1}^n \psi\left(v_i\right) = \sum_{i=1}^n \psi\left(v_i'\right),
	\end{align} 
	holds if and only if there exists a permutation $\pi \in S_n$ such that $(v_1',\ldots,v_n') = (v_{\pi(1)},\ldots,v_{\pi(n)})$. For fixed $m$ and $n$, the class of all such functions is denoted by $\Psi_{m,n}$. Moreover, we refer to $p$, the dimension of the co-domain of the function $\psi$ as $p({\psi})$.
\end{definition}

The summation of the function $\psi$ aims to generate an algebraic form for a multiset of vectors. Previous works \cite{zaheer2017deepset,xu2018how} developed ideas to translate multisets to functional forms. However, the approaches in these works cannot translate a multiset of ``vectors'' of arbitrary dimension to an algebraic summation which is required by \cref{def: cef}. In \cref{prop: cef}, we introduce candidate multiset-equivalent functions for every $m$ and $n$ to ensure that the existence of such functions and provide a constructive framework for the proofs in this paper. The term ``multise'' in MEF refers to a generalisation of a set in which repetition of elements is permitted (see \cref{app: notation} for a formal definition of a multiset). The name multiset-equivalent function for $\psi$ in \cref{def: cef}, relates to the fact that the summation of $\psi$ over a sequence of vectors preserves all the data up to a permutation and thus this sum is equivalent to the ``multiset'' of data. It is not trivial to find an MEF. For instance, note that the identity function (which leads to $\sum_{i=1}^n v_i = \sum_{i=1}^n v_i'$) is not an MEF for $n>1$. This is because when $m=1$ and $n=2$, we have $2+5=3+4$ but $(2,5)$ is not a permutation of $(3,4)$. A natural question here is whether such function exists at all. The following proposition introduces candidate elements of $\Psi_{m,n}$ for all positive integers $n$ and $m$.

\begin{proposition}\label{prop: cef}
	The followings are specific constructions of MEFs for (i) $m=1$, and (ii) $m > 1$:
	\begin{enumerate}[(i)] 
		\item \label{psi1}Consider the function $\psi:\R \to \R^{n}$ such that for $v\in \R$, $\psi(v) = (v, v^2, \ldots, v^n)$. Then $\psi\in \Psi_{1,n}$. Moreover, note that $p(\psi)=n$.
		\item \label{psi2}Let $m>1$ and consider $v = (v_1,\ldots,v_m)\in \R^m$. Define $\psi(v)=\psi(v_1, \cdots, v_m)$ to be an $n\times m \times m$ array (tensor) with real elements such that for every $\ell\in [n]$ and $r,s\in[m]$ with $r<s$: 
		\begin{align}\label{a14}
			[\psi(v_1, \cdots, v_m)]_{\ell,r,s} &= \operatorname{Re}\left(\left(v_r+v_s\sqrt{-1}\right)^{\ell}\right),\\
			\label{a15}	[\psi(v_1, \cdots, v_m)]_{\ell,s,r} &= \operatorname{Im}\left(\left(v_r+v_s\sqrt{-1}\right)^{\ell}\right).
		\end{align}
		Note that $\psi(v) \in \R^{n\times m\times m}$. Let us re-shape $\psi(v)$ into a long vector in $\R^{m^2n}$, and consider $\psi: \R^m \to \R^{m^2n}$. Then $\psi\in \Psi_{m,n}$. Moreover, note that $p(\psi)=m^2n$. 
	\end{enumerate}
\end{proposition}
To construct valid MEFs, the idea behind this specific choice of $\psi$ in \cref{prop: cef} for $m=1$ is that when we take the sum $\psi$ over $n$ scalars, it encodes them into the roots of a unique polynomial and thus it preserves the data up to a permutation. For $m>1$, this idea is extended by encoding vectors of arbitrary size into the roots of a system of complex polynomials. See the proof in \cref{app: prop-cef}.\\

\noindent{\bf(ii) Constructing a basis function based on MEFs.} In the following definition, we construct a graph function over $\G_{n,d}$ through MEFs which we call a {\it basis} function.
\begin{definition}[Basis Function]\label{def: beta}
	Define the graph function $\mathcal{B} = (\beta_1,\ldots,\beta_n)$ over $\G_{n,d}$ such that for $i\in[n]$: 
	\begin{align}\label{ab2-app}
		\beta_i\left(W,X\right) = \left(x_i,\sum_{j\in [n]_{-i}}\psi_2\left(x_j,w_{i,j},\sum_{\ell\in [n]_{-j}}\psi_1\left(x_{\ell},w_{j,\ell}\right)\right)\right),
	\end{align}
	where $\psi_1\in \Psi_{d+1,n-1}$ and $\psi_2\in \Psi_{p(\psi_1)+d+1,n-1}$ is the same for all $i\in[n]$. We call $\mathcal{B}$  a basis function over $\G_{n,d}$. 
\end{definition}
This specific structure of a basis function leads to an important property which is stated and formally proven in the following proposition.  
\begin{proposition}\label{prop: beta}
	Let $\mathcal{B} = (\beta_1,\ldots,\beta_n)$ be a basis function over $\G_{n,d}$. Then $\mathcal{B}\in \F_{n,d}$ with the following additional property:
	Given two graphs $G=([n],W,X)$ and $G'=([n],W',X')$ in $\tilde{\G}_{n,d}$, for every $i\in [n]$, if $\beta_i(W,X)=\beta_i(W',X')$, then there exists $\pi\in \nabla_i$ such that $W' = \sigma_{\pi}(W)$ and $X' = \lambda_{\pi}(X)$.
\end{proposition}
The fact that $\psi_1$ and $\psi_2$ are MEFs is crucial in showing that the specific structure of the basis function in \eqref{ab2-app} leads to \cref{prop: beta}. We omit the details here and refer to the proof of \cref{prop: beta} in \cref{app: prop-beta}.\\

\noindent{\bf(iii) Representing any permutation-compatible function in terms of the basis function.} The key property of $\beta_i$ mentioned in part (ii) enables us to represent any permutation-compatible function $F$ in terms of the basis function. This is formalized in the following theorem.
\begin{theorem}[Main Representation Theorem] \label{theo: rho} 
	Suppose $F\in \mathcal{F}_{n,d}$ with $F = (f_1,\ldots,f_n)$ and let $\mathcal{B} = (\beta_1,\ldots,\beta_n)$ be a basis function over $\G_{n,d}$ and recall $\tilde{\G}_{n,d}$ from \cref{def: G-tilde}. Then, there exists a function $\rho$ s.t. for every $i\in [n]$ and $G=([n],W,X)\in \tilde{\G}_{n,d}$, we have $f_i(W,X) = \rho(\beta_i(W,X))$. 
\end{theorem}
\cref{theo: rho} states that any node function $f_i$ of a {\GPC} function $F$ can be written in terms of $\beta_i$ over the set of graphs with distinct features $\tilde{\G}_{n,d}$ defined in \cref{def: G-tilde}. \cref{theo: rho} is an equivalent way of saying that subject to having distinct node features in a graph, $\beta_i(W,X)=\beta_i(W',X')$ leads to $f_i(W,X)=f_i(W',X')$. \\

\noindent{\bf(iv) Constructing the GNN.} Using the Representation \cref{theo: rho}, to generate $F$, it suffices to construct a GNN such that $h_i^{(k)}(W,X) = \rho(\beta_i(W,X))$. Due to the construction of $\beta_i$ in \eqref{ab2-app}, a 
GNN can generate it in two iterations. Using $\rho$ in the third iteration then completes the construction. More formally, we set candidates for $\phi_1$, $\phi_2$, and $\phi_3$ (defined in \eqref{gnn-formula}) as follows:
\begin{align}
	\label{aj0}&\phi_1\left(h_j^{(0)},h_{\ell}^{(0)},w_{j,\ell}\right)=\left(\frac{1}{n-1}h_j^{(0)},\,\,\psi_1\left(h_{\ell}^{(0)},w_{j,\ell}\right)\right),\\
	\label{aj1-2}&\phi_2\left(h_i^{(1)},h_j^{(1)},w_{i,j}\right)=\left(\frac{1}{n-1}\left[h_i^{(1)}\right]_{1:d},\psi_2\left(\left[h_j^{(1)}\right]_{1:d},w_{i,j},\left[h_j^{(1)}\right]_{d+1:d+p(\psi_1)}\right)\right),\\
	\label{aj2}&\phi_3\left(h_i^{(2)},h_j^{(2)},w_{i,j}\right)=\frac{1}{n-1}{\rho\left(h_i^{(2)}\right)}.
\end{align}
In \cref{aj1-2}, the notation $[v]_{a:b}$ for  $v=(v_1,\ldots,v_m)$ means $[v]_{a:b}=(v_a,v_{a+1},\ldots,v_b)$. It is straightforward to see that $h_i^{(2)}(W,X)=\beta_i(W,X)$ and $h_i^{(3)}(W,X)=\rho(\beta_i(W,X))$. This results in $h_i^{(3)}(W,X)=f_i(W,X)$ for all $G=([n],W,X)\in \tilde{\G}_{n,d}$, due to \cref{theo: rho}.

The steps (i) to (iv) provide a proof sketch for \cref{theo: distinct} which is the main stand to reach the other results in this paper. At the heart of this analysis lies \cref{theo: rho} which has other theoretical benefits. For example, one might ask if we can generate a continuous permutation-compatible  graph function by using continuous $\phi_k$'s in the GNN? In particular, answering this question is useful when one chooses $\phi_k$ from the class of multi-layer perceptrons (MLPs) as good approximates for continuous functions. The following result provides  an answer. 
\begin{corollary}\label{corr: cont}
	Considering \cref{theo: distinct}, if $F(W,X)$ is continuous with respect to $(W,X)$, then a GNN $H^{(k)}$ with continuous inner functions $\phi_k$ exists that works for the theorem.
\end{corollary}

\section{Feature Crafting to Generate All Graph Functions}\label{sec: beyond}

In this section, we discuss how GNNs can go beyond permutation compatibility and generate any graph function. In brief, we show that  if we augment the identity of each node $i$ to its associated feature $x_i$, i.e. set $\tilde{x}_i=(x_i,i)$, and let $\tilde{X}$ be the concatenation of $\tilde{x}_i\,$s, then for any graph function $F(W,X)$, a GNN exists that receives $W$ and $\tilde{X}$ and outputs $F(W,X)$.
\begin{theorem}\label{theo: arbitrary}
	Suppose $n$ fixed distinct vectors $y_1,\ldots,y_n\in \R^{d_0}$ are given and we augment them to features of all graphs. More formally, for every $G=([n],W,X)\in\G_{n,d}$ with feature matrix $X=(x_1^{\top},\ldots,x_n^{\top})$, we augment $y_i$ to the feature $x_i$ to construct $\tilde{x}_i=(x_i,y_i)\in \R^{d+d_0}$ for all $i\in[n]$. One simple option is $d_0=1$ and $(y_1,\ldots,y_n)=(1,\ldots,n)$. Let $\tilde{X}=(\tilde{x}_1^{\top},\ldots,\tilde{x}_n^{\top})$. Then for every graph function $F(W,X)$, there exists a GNN $H^{(k)}$ with a finite $k\geq 0$ over $\G_{n,d+d_0}$ such that $H^{(k)}(W,\tilde{X}) = F(W,X)$ for all $G=([n],W,X)\in\G_{n,d}$.
\end{theorem}
The proof of \cref{theo: arbitrary} is built on the framework of basis functions described in Section~\ref{sec: char}. In brief,  the basis function output  $\beta_i(W,\tilde{X})$  uniquely determines the triple $(W,X,i)$ and thus a GNN can achieve any graph function as the next step after generating $\beta_i(W,\tilde{X})$.\\

\noindent{\bf How do the augmentations in \cref{theo: arbitrary p.c} and \cref{theo: arbitrary} differ?} We described two types of augmentation in \cref{theo: arbitrary p.c} and \cref{theo: arbitrary} which we call {\it soft-coded} and {\it hard-coded} augmentation, respectively. In both cases, distinct nodes in a graph receive distinctly augmented values. However, in the hard-coded case, a fixed and unique value is augmented to the feature of node $i$ for \emph{all} the graphs in $\G_{n,d}$, while this is not necessarily the case for a soft-coded augmentation. This difference  can be stated in  logical terms as follows:
\begin{align}
	\label{hard}\text{hard-coded}:&\quad \exists y_1,\ldots,y_n \quad \forall G\in \G_{n,d}\quad \text{$G$ is augmented by $y_i\,$s,}\\
	\label{soft}\text{soft-coded}:&\quad \forall G\in \G_{n,d} \quad \exists y_1,\ldots,y_n\quad \text{$G$ is augmented by $y_i\,$s}.
\end{align}
To see the computational difference between these two augmentations, consider the example illustrated in \cref{fig: hard-soft} -(a) and (b). In each of the \cref{fig: hard-soft} -(a) and (b), the right graph is obtained by swapping the labels $1$ and $2$ in the left graph. A hard-coded augmentation means $(y_1,y_2,y_3,y_4)=(1,2,3,4)$ for both labelings of the graph. This is sensitive to node labeling. In other words, omitting the node labels, one sees two different sets of node features for the same graph in \cref{fig: hard-soft} -(b). 
In contrast, under a soft-coded augmentation, we can have $(y_1,y_2,y_3,y_4)=(1,2,3,4)$ for the left graph in \cref{fig: hard-soft} -(a) and $(y_1,y_2,y_3,y_4)=(2,1,3,4)$ for the right graph. Unlike the hard-coded augmentation, by ignoring the labels, we see the same set of node features for the same graph. Hence, the soft-coded augmentation can be set independently of node labeling. This elaboration reveals that, under a hard-coded augmentation, building a full dataset for training the GNN needs potentially $n!$ samples corresponding to all the $n!$ possible labelings of the same graph. This is the same cost when one treats a graph data as a matrix pair $(W,X)$ and gives it to an ordinary feed-forward neural net, ignoring the graph-based structure of the GNN.
\begin{figure*}[t]
	\centering
	\subfigure[Soft-coded\label{fig: soft}  ]{\includegraphics[width=0.23\linewidth]{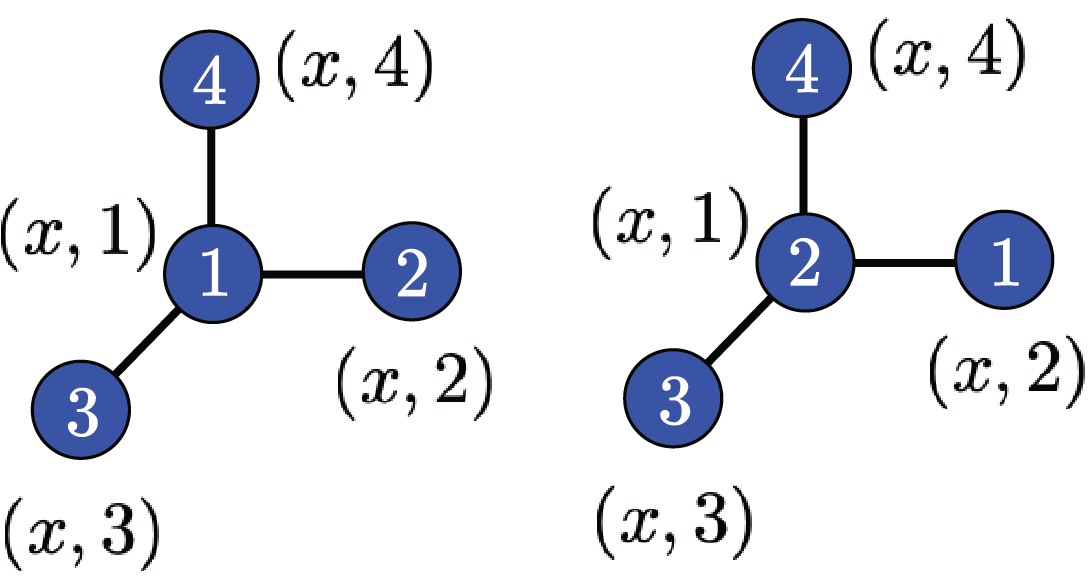}} \hspace{0.07cm}
	\subfigure[Hard-coded\label{fig: hard} ]{\includegraphics[width=0.23\linewidth]{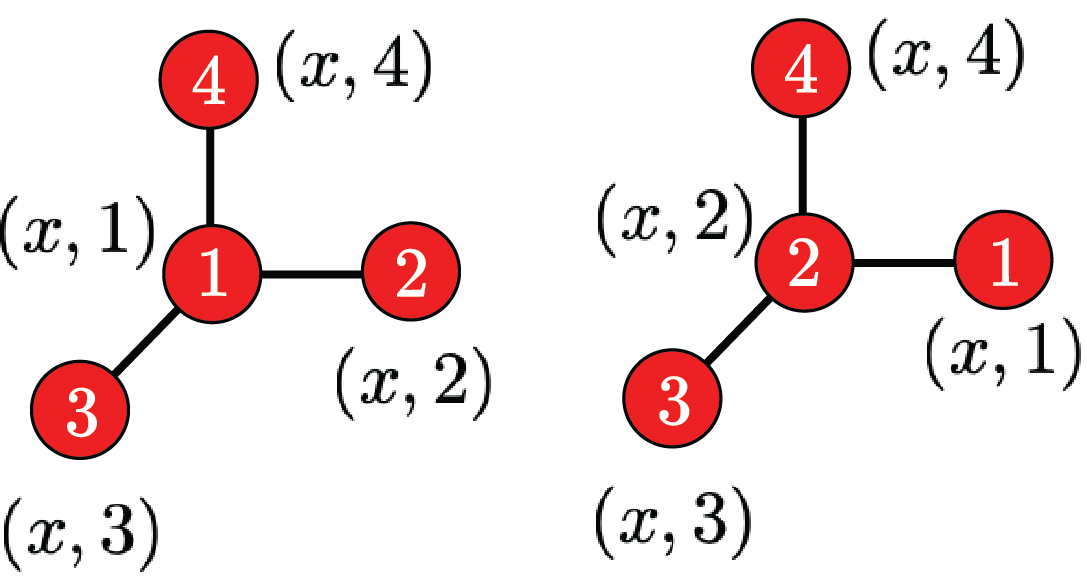}} \hspace{0.07cm}
	\subfigure[Identical labels ]{\includegraphics[width=0.23\linewidth]{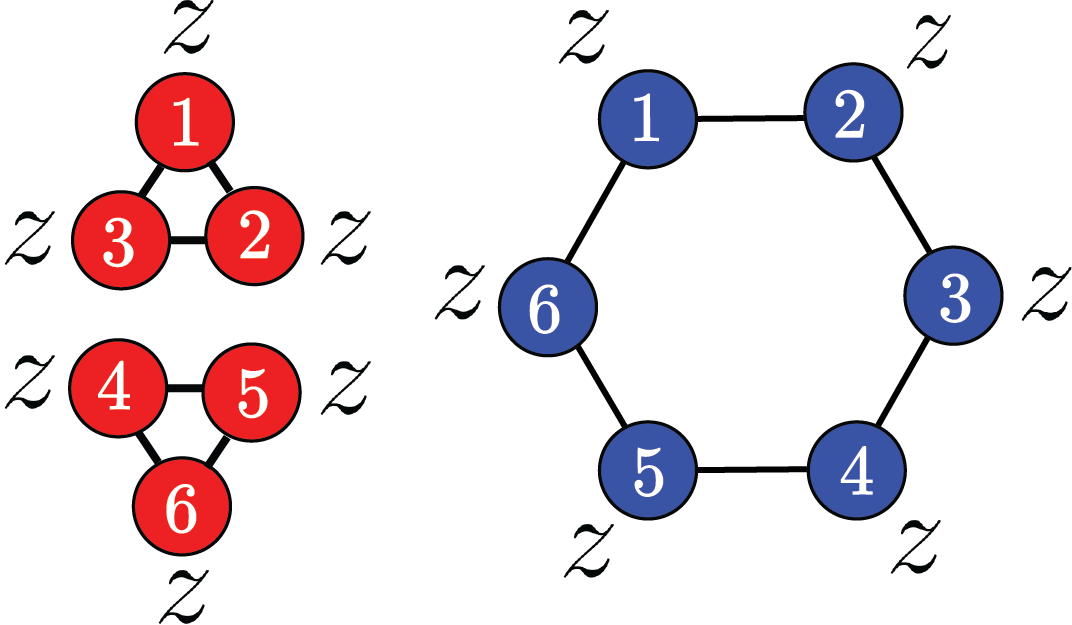}} \hspace{0.07cm}
	\subfigure[Non-identical labels.]{\includegraphics[width=0.23\linewidth]{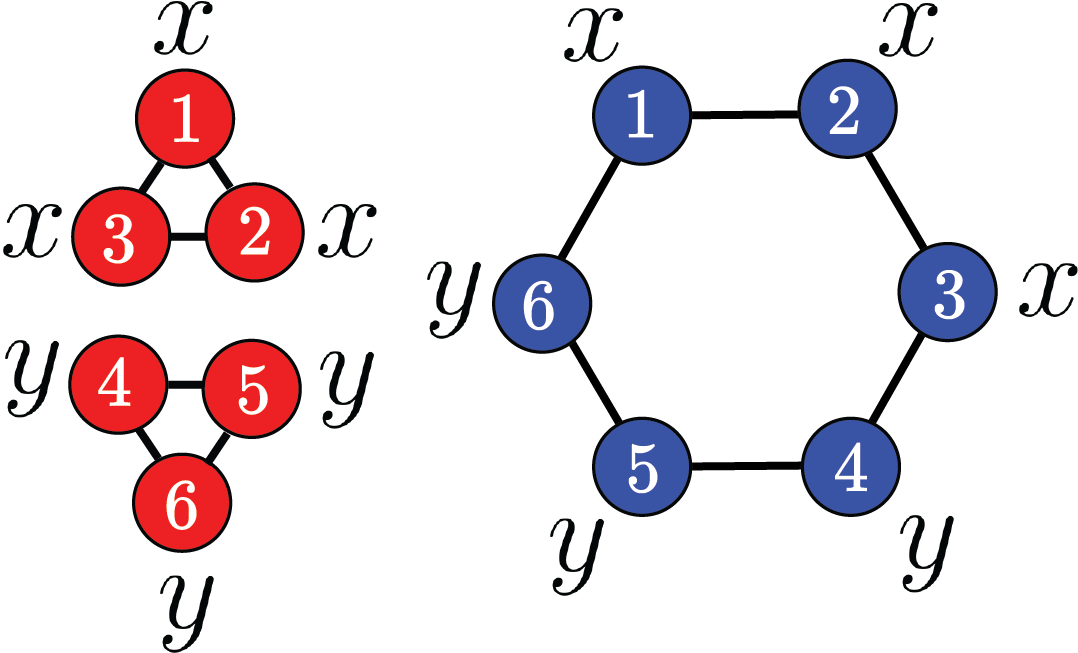}} 
	\caption{(a) and (b): The difference between the soft-coded and hard-coded augmentations. (c) and (d): Illustration of the hexagon versus the two-triangle graph, where 1-WL cannot separate between the two under identical labels (c) and can separate them under non-identical labels (d).}
	\label{fig: hard-soft}
\end{figure*}

From \cref{sec: intuition}, permutation compatibility translates into independency  from node identities. Moreover, we know that hard-coded augmentation essentially means revealing node identities to the GNN. In this way, \cref{theo: arbitrary} requires  revealing all the node identities to the GNN via a hard-coded augmentation. From the discussion above, we know that this is costly. Now the question is: In the case that $F$ depends only on a subset of node identities, can we use a hard-coded augmentation only on that subset of nodes instead of all the nodes? The following informally-stated corollary gives a positive answer. We will see the implications of this result on the shortest-path problem later in Section~\ref{sec: SP}. For a formal statement and proof, we refer to \cref{app: corr-arbitrary}.
\begin{corollary}[Informal]\label{corr: hybrid}
	Suppose a graph function $F$ does not depend on the node labels except for $i_1,\ldots,i_k\in [n]$, i.e., \eqref{a23} holds for every $\pi\in S_n$ that satisfies $\pi(j)=j$ for $j\in \{i_1,\ldots,i_k\}$. Then $F$ can be generated by a GNN under a hard-coded augmentation for $i_1,\ldots,i_k$ and a soft-coded augmentation for other nodes.
\end{corollary}
\section{Weisfeiler-Lehman Test Versus GNN}\label{sec: wl-connection}
\looseness=-1 Recent works  \cite{xu2018how,morris2019weisfeiler} have explained the expressiveness of GNNs via the Weisfeiler-Lehman (WL)  isomorphism test. We now discuss the connection between permutation compatibility and WL test. 

Starting with all nodes of identical labels/colors, the 1-WL test iteratively and through message passing assigns new labels to nodes in the form of multi-sets. If at any iteration of the procedure, the labeling of two graphs differ, they are certainly not isomorphic. However, it can very well happen that the 1-WL test produces the same labeling at every single iteration while the two graphs are not isomorphic, e.g., the hexagon and the two-triangle graph shown in \cref{fig: hard-soft}-(c). Similarly, a GNN that aims to compute the min-cut function (an instance of a permutation-compatible function) produces the same value for both graphs if it starts with identical node features. In fact, this is a general phenomenon: if the 1-WL fails  then a permutation-compatible function with identical features cannot be generated by a GNN. To formalize this equivalency, which is essentially the same result as in \cite{xu2018how,morris2019weisfeiler}, we need to set a notation for graphs under identical node features and also a formal proof in our general setting. To this end, consider the following definition.
\begin{definition}
	Fix a constant vector $c\in \R^d$. Let us define
	\begin{align*}
		\mathcal{G}^{c}_{n,d} = \left\{G=([n],W,C)\in \mathcal{G}_{n,d} \mid W\in \{0,1\}^{n \times n}, C=(c^{\top},\ldots,c^{\top})\right\}.
	\end{align*}
\end{definition}
Moreover, we need to precisely define what it means that a GNN can not separate between two graphs. The following definition specifies this notion using the notation $\#\{\cdot\}$. This notation refers to a multiset of elements. A multiset generalizes the concept of a set by allowing the repetition of elements. For a formal definition, we refer to \cref{app: notation}. 
\begin{definition}
	Suppose $G_1,G_2 \in \mathcal{G}_{n,d}$, with $G_1=([n],W_1,X_1)$ and $G_2=([n],W_2,X_2)$. We say that GNNs over $\mathcal{G}_{n,d}$ cannot separate $G_1$ and $G_2$ if for every $k\geq0$ and every GNN $H^{(k)}=(h_1^{(k)},\ldots,h_n^{(k)})$, we have $\#\{h_i^{(k)}(W_1,X_1) \mid i\in [n]\} = \#\{h_i^{(k)}(W_2,X_2) \mid i\in [n]\}$.
\end{definition}
 The following result formalises the earlier statement. 
\begin{proposition}
	\label{theo: wl}
	Suppose $G_1,G_2 \in \mathcal{G}^{c}_{n,d}$ are graphs without isolated nodes. Then GNNs over $\mathcal{G}^{c}_{n,d}$ cannot separate $G_1$ and $G_2$ if and only if 1-WL cannot distinguish between $G_1$ and $G_2$.  
\end{proposition}
In light of the above theorem, feature augmentation for GNNs in general is unavoidable as there are permutation compatible functions (such as min-cut) that cannot be generated by any GNN under identical node features.  
\section{Dynamic Programming Versus GNN}\label{sec: dp}
Dynamic Programming (DP) is an iterative mechanism that evolves the state of some entities by starting at initial states and updating the current state of each entity as a function of the current state of others. Treating $h_i^{(k)}$ as the state of node $i$ in the $k$-the iteration, GNNs also lie in this category. Therefore, one would expect a close connection between GNN and DP. One possible approach to explain the connection between GNN and DP is to quantify the connection between their iterative structure \cite{xu2020what,dudzik2022graph,velickovic2020Neural}. However, our results are of different nature. For any algorithmic procedure on graph, DP or otherwise, for which there is an output graph function $F$, we discuss how GNNs can  generate $F$. Due to \cref{theo: arbitrary}, this is possible for any graph function $F$. The only thing to consider further is that if $F$ is permutation compatible or if it depends on identity of only a subset of nodes (as formalised in \cref{corr: hybrid}), we can avoid the costly hard-coded augmentation of \cref{theo: arbitrary} and use \cref{theo: arbitrary p.c} or \cref{corr: hybrid} instead. Hence, given $F_{\operatorname{DP}}(W,X)$ as the output of a DP, based on whether $F_{\operatorname{DP}}$ is permutation compatible or otherwise, we can generate it under a proper augmentation using \cref{theo: arbitrary p.c}, \cref{theo: arbitrary}, and \cref{corr: hybrid}. As a particular example, let us explain the situation for the shortest path problem in \cref{sec: SP}.

\subsection{Shortest-Path-Length Problem}\label{sec: SP}
In connection with DP, in this section, we consider the shortest-length problem as the output of the Bellman-Ford dynamic program. We show that GNNs are able to generate the shortest-path-length function to a source node as long as the source node is identified through the node features. This identification of the source node is trivially required by any algorithm. For a formal argument, consider the distance-to-node-$1$ function $F$ in \cref{node1} of \cref{ex: funcs}. Note that $F$ is not permutation compatible since it does not necessarily satisfy \eqref{a23} for a $\pi$ with $\pi(1)\neq 1$. However, \eqref{a23} is satisfied over all permutations $\pi$ s.t.  $\pi(1)= 1$. This is formalised in the following lemma.
\begin{lemma}\label{lem: SP}
	Let $F(W,X)$ be the distance-to-node-1 function defined in \cref{node1} of \cref{ex: funcs}. Since $F(W,X)$ ignores $X$, we use the notation $F(W)=(f_1(W),\ldots,f_n(W))$. Then for every $\pi\in \nabla_1$ and $i \in [n]$, we have $f_{\pi{(i)}} (W) = f_i(\sigma_{\pi}(W))$ for every weight matrix $W$. 
\end{lemma}
Based on \cref{lem: SP}, \cref{corr: hybrid} implies that hard-coded augmentation is only needed on node $1$ to generate the $F$ which equivalently means revealing the identity of node $1$ to the GNN. \cref{corr: hybrid} also requires a soft-coded augmentation on other nodes. The latter, however, turns out to be unnecessary due to following proposition which summarises our results on the shortest path problem. 
\begin{proposition}[Informal]\label{prop: SP}
	Letting $F(W)$ to be the distance-to-node-1 function, (i) $F$ is not permutation compatible. (ii) Using a fixed feature matrix $X_0=(y_1,\ldots,y_n)$ for all graphs, a GNN can generate $F$ if and only if $y_1$ is distinct from other $y_i\,$s.
\end{proposition}
See \cref{app: prop-SP} for a formal statement and proof. Note that $X_0=(1,0,\ldots,0)$ works well for \cref{prop: SP} while for example $X_0=(1,1,0,\ldots,0)$ fails. The latter is intuitively trivial since it gives no clue to the GNN in identifying the source node.

\section*{Conclusion} 
In this paper, we provided an analytic framework to study the representation power of GNNs. We introduced the fundamental notion of permutation compatibility that fully characterizes what graph functions may (or may not) be generated by a GNN.   

\section*{Acknowledgement} 
We would like to thank Javid Dadashkarimi,  Petar Veli\u{c}kovi\'{c}, and Amin Saberi for their comments and discussions that led to the current version. 

The work of M. Fereydounian and H. Hassani is funded by DCIST, NSF CPS-1837253, and NSF CIF-1943064 and  NSF CAREER
award CIF-1943064, and Air Force Office of Scientific Research Young Investigator Program
(AFOSR-YIP) under award FA9550-20-1-0111. 

The research of A. Karbasi is supported by  NSF (IIS-1845032), ONR (N00014-
19-1-2406), and the AI Institute for Learning-Enabled Optimization at Scale (TILOS).

\bibliography{ref}
\bibliographystyle{plain}

\newpage
\appendix
\onecolumn


\section{Preliminaries for appendices}\label{app: notation}
\noindent{\bf Notation.} We denote by $f\circ g$, the composition of functions $f$ and $g$, meaning that $f\circ g(x) = f(g(x))$. Moreover, for a vector $v=(v_1,\ldots,v_m)$, let $[v]_{a:b}$ for $a\leq b$, denote the sub-vector consisting of the elements with indices starting from $a$ to $b$, that is, $[v]_{a:b}=(v_a,v_{a+1},\ldots,v_b)$. In this notation, we also use ``end'' to refer to the last index, i.e., $[v]_{a:\text{end}}=[v]_{a:m}$.

In the following, we provide formal definitions that are used in the proofs. We start by the formal definition of a multiset. 
\begin{definition}[Multiset]
	Multisets generalize the concept of a set in which the repetition of elements is allowed. A multiset is a pair $Y=(S,m)$, where $S$ is the underlying set of the distinct elements of $Y$ and $m:S\to \mathbb{Z}_{\geq{1}}$ is the function that indicates the multiplicity of each element.
\end{definition}
For the ease of explanation, we set the following notation for a multiset.
\begin{definition}[Notation $\#\{\cdot\}$]\label{def: hash}
	Suppose $y_1,\ldots,y_n$ are some objects with possibly repeated elements, then the multiset containing $y_1,\ldots,y_n$ is denoted by $\#\{y_1,\ldots,y_n\}$. More formally, if the set of distinct elements among $y_1,\ldots,y_n$ is $\{y_{i_1},\ldots,y_{i_r}\}$ with $m(y_{i_1}),\ldots,m(y_{i_n})$ denoting the multiplicity of the elements, then 
	\begin{align}
		\#\{y_1,\ldots,y_n\} = (\{y_{i_1},\ldots,y_{i_r}\},m).
	\end{align}
	The notation $\#\{\cdot\}$ considers the repetition but ignores the order of the elements. Therefore, for the sequences $(y_1,\ldots,y_n)$ and $(z_1,\ldots,z_m)$ of possibly repeated elements, the equation
	\begin{align}
		\#\{y_1,\ldots,y_n\} = \#\{z_1,\ldots,z_m\}
	\end{align}
	holds if and only if $m=n$ and there exists a permutation $\pi\in S_n$ such that $\left(y_1,\ldots,y_n\right) = \left(z_{\pi(1)},\ldots,z_{\pi(n)}\right)$.
\end{definition}

\section{Formal statement and proof of \cref{prop: informal}}\label{sec: replacing}
To formalise a GNN with aggregator operator \eqref{AGG}, we define the Extended-GNN analogous to \cref{def: gnn} as follows.
\begin{definition}[Extended-GNN]\label{def: gnn2}
	An Extended Graph Neural Network (Extended-GNN) is an iterative mechanism that generates a sequence of functions $E^{(k)}$, $k\geq 0$, over $\G_{n,d}$ in the following manner.  For $G=([n],W,X)\in\G_{n,d}$, the function $E^{(k)}(W,X)=(e_1^{(k)}(W,X),\ldots,e_n^{(k)}(W,X))$ is given as 
	\begin{align}
	&\text{if }k=0:\, e_i^{(0)} = x_i,\\
	\label{gnn-formula2}&\text{if }k\geq 1:\,e_i^{(k)}=\Phi_k\left(e_i^{(k-1)},\#\left\{(e_j^{(k-1)},w_{i,j})\mid j\in[n]_{-i}\right\}\right),
	\end{align}
	for some functions $\Phi_k,\,k\geq 1$, where for each $k$, the outputs of $\Phi_k$ lie in some Euclidean vector space. 
\end{definition} 
\begin{proposition}[Formal]\label{prop: extended gnn}
	The function-class of GNNs is equivalent to the function-class of Extended-GNNs. More formally, suppose a graph function $F$ over $\G_{n,d}$ is given. Then, there exists a GNN, denoted by $H^{(k')}$, over $\G_{n,d}$ such that $H^{(k')}(W,X)=F(W,X)$ for all $G=([n],W,X)\in\G_{n,d}$ if and only if there exists an Extended-GNN $E^{(k)}$ over $\G_{n,d}$ such that $E^{(k)}(W,X)=F(W,X)$ for all $G=([n],W,X)\in\G_{n,d}$.
\end{proposition}
\begin{proof}[{\bf Proof of \cref{prop: extended gnn}}]
	First note that the aggregator \eqref{gnn-formula} is a special case of \eqref{gnn-formula2}. Therefore, the GNN defined in \cref{def: gnn} is a special case of the Extended-GNN defined in \cref{def: gnn2} and thus one side of the claim is immediate. To prove the other direction, suppose an Extended-GNN $E^{(k)}(W,X)$ over $G_{n,d}$ is given. We show that there exists a GNN $H^{(k')}(W,X)$ such that for every $r\geq0$: $E^{(r)}(W,X)=H^{(2r)}(W,X)$ for all $G=([n],W,X)\in\G_{n,d}$. To this end, let us set some notations. Consider $E^{(k)}=(e_1^{(k)},\ldots,e_n^{(k)})$ and suppose the outputs of $e_i^{(k)}$ lie in $\R^{\alpha_k}$ for some $\alpha_k$ and fix multiset-equivalent functions (MEFs) $\psi_k\in \Psi_{\alpha_{k-1}+1,n-1}$, defined in \cref{def: cef}. Also note that candidates for MEFs are provided in \cref{prop: cef}. 
	
	As the first step, we claim that for all $k\geq 1$, there exists a function $\Theta_k$ such that 
	\begin{align}\label{aat}
	\Phi_k\left(e_i^{(k-1)},\#\left\{(e_j^{(k-1)},w_{i,j})\mid j\in[n]_{-i}\right\}\right) = \Theta_k\left(e_i^{(k-1)},\sum_{j\in [n]_{-i}}\psi_k\left(e_j^{(k-1)},w_{i,j}\right)\right).
	\end{align}
	To prove \eqref{aat}, it suffices to show that having
	\begin{align}\label{aat2}
	\left(e_i^{(k-1)},\sum_{j\in [n]_{-i}}\psi_k\left(e_j^{(k-1)},w_{i,j}\right)\right)=\left(\hat{e}_i^{(k-1)},\sum_{j\in [n]_{-i}}\psi_k\left(\hat{e}_j^{(k-1)},\hat{w}_{i,j}\right)\right)
	\end{align}
	leads to
	\begin{align}\label{aat3}
	\Phi_k\left(e_i^{(k-1)},\#\left\{(e_j^{(k-1)},w_{i,j})\mid j\in[n]_{-i}\right\}\right) = \Phi_k\left(\hat{e}_i^{(k-1)},\#\left\{(\hat{e}_j^{(k-1)},\hat{w}_{i,j})\mid j\in[n]_{-i}\right\}\right).
	\end{align}
	To show that \eqref{aat2} leads to \eqref{aat3}, note that from \eqref{aat2}, we have $e_i^{(k-1)}=\hat{e}_i^{(k-1)}$, and 
	\begin{align}\label{aat5}
	\sum_{j\in [n]_{-i}}\psi_k\left(e_j^{(k-1)},w_{i,j}\right)=\sum_{j\in [n]_{-i}}\psi_k\left(\hat{e}_j^{(k-1)},\hat{w}_{i,j}\right). 
	\end{align} 
	Having \eqref{aat5}, the definition of an MEF (see \cref{def: cef}) implies that  
	\begin{align}\label{aat4}
	\#\left\{(e_j^{(k-1)},w_{i,j})\mid j\in[n]_{-i}\right\} = \#\left\{(\hat{e}_j^{(k-1)},\hat{w}_{i,j})\mid j\in[n]_{-i}\right\}.
	\end{align}
	\cref{aat4} together with $e_i^{(k-1)}=\hat{e}_i^{(k-1)}$ proves \eqref{aat3}. Hence, we showed the existence of the function $\Theta_k$ satisfying \eqref{aat}.

	Next, we construct a GNN to generate a given Extended-GNN $E^{(k)}$, using the function $\Theta_k$ introduced above. To this end, given an Extended-GNN $E^{(k)}=(e_1^{(k)},\ldots,e_n^{(k)})$, we construct a GNN $H^{(k')}=(h_1^{(k')},\ldots,h_n^{(k')})$ such that $h_i^{(2r)}=e_i^{(r)}$ for all $r\geq0$ and $i\in [n]$. This construction is as follows: For $r\geq 1$, set
	\begin{align}
	\label{aw1}\phi_{2r-1}\left(h_i^{(2r-2)},h_{j}^{(2r-2)},w_{i,j}\right) &= \left(\frac{1}{n-1}h_i^{(2r-2)},\psi_r\left(h_{j}^{(2r-2)},w_{i,j}\right)\right),\\
	\label{aw2}\phi_{2r}\left(h_i^{(2r-1)},h_{j}^{(2r-1)},w_{i,j}\right) &=\frac{1}{n-1}\Theta_r\left(h_i^{(2r-1)}\right).
	\end{align}
	Due to \eqref{aw1} and \eqref{aw2}, for all $i\in [n]$, we have
	\begin{align}
	\label{aq1}h_{i}^{(2r-1)} &= \sum_{j\in [n]_{-i}} \phi_{2r-1}\left(h_i^{(2r-2)},h_{j}^{(2r-2)},w_{i,j}\right) =\left(h_i^{(2r-2)},\sum_{j\in [n]_{-i}}\psi_r\left(h_{j}^{(2r-2)},w_{i,j}\right)\right),\\
	\label{aq2}h_{i}^{(2r)} &= \sum_{j\in [n]_{-i}} \frac{1}{n-1}\Theta_r\left(h_i^{(2r-1)}\right)=\Theta_r\left(h_i^{(2r-2)},\sum_{j\in [n]_{-i}}\psi_r\left(h_{j}^{(2r-2)},w_{i,j}\right)\right).
	\end{align}
	Now, using the relation between $\Theta_r$ and $\Psi_r$ in \eqref{aat}, we conclude from \eqref{aq2} that
	\begin{align}\label{389}
	h_{i}^{(2r)} = \Phi_r\left(h_i^{(2r-2)},\#\left\{(h_j^{(2r-2)},w_{i,j})\mid j\in[n]_{-i}\right\}\right).
	\end{align}
	Next, we claim that $h^{(2r)}_i=e_i^{(r)}$ for all $r\geq 0$ and $i\in [n]$. We prove this by induction on $r$. For $r=0$, we have $h^{(0)}_i=e_i^{(0)}=x_i$.  Given the induction hypothesis for $r-1$, we have $h^{(2r-2)}_i=e_i^{(r-1)}$ for all $i\in [n]$. Replacing this into \eqref{389} leads to
	\begin{align}
	h_{i}^{(2r)} = \Phi_r\left(e_i^{(r-1)},\#\left\{(e_j^{(r-1)},w_{i,j})\mid j\in[n]_{-i}\right\}\right)=e_i^{(r)}.
	\end{align}
	Hence, we have $E^{(r)}(W,X)=H^{(2r)}(W,X)$ for all $r\geq 0$ and all $G=([n],W,X)\in\G_{n,d}$.
\end{proof}

\section{Proof of Theorem~\ref{theo: necessity}}\label{app: theo-necessity}
\begin{proof}
	Consider a GNN  $H^{(k)}=(h_1^{(k)},\ldots,h_n^{(k)})$ for some $k\geq 0$ over $\G_{n,d}$. We want to show that $H^{(k)}\in \mathcal{F}_{n,d}$. To this end, we must show that for every given $\pi\in S_n$, the following holds for all $G=([n],W,X)\in \G_{n,d}$ and all $i \in [n]$:
	\begin{align}
	\label{ht}h_{\pi(i)}^{(k)}(W,X) = h_{i}^{(k)}\left(\sigma_{\pi}(W),\lambda_{\pi}(X)\right).
	\end{align}
	We prove \eqref{ht} by induction on $k$. For $k=0$, we have $h_i^{(0)}(W,X)=x_i$ for all $i\in [n]$. Hence,
	\begin{align}
	h_i^{(0)}\left(\sigma_{\pi}(W),\lambda_{\pi}(X)\right) = x_{\pi(i)} = h_{\pi(i)}^{(0)}\left(W,X\right).
	\end{align}
	Given the induction hypothesis for $k-1$, we must show that \eqref{ht} holds for $k$. From \eqref{gnn-formula}, we have
	\begin{align}
	h_{\pi(i)}^{(k)}(W,X) &= \sum_{\ell\in [n]_{-\pi(i)}} \phi_{k}\left(h_{\pi(i)}^{(k-1)}\left(W,X\right),h_{\ell}^{(k-1)}\left(W,X\right),w_{\pi(i),\ell}\right)\\
	\label{a1}    &= \sum_{j\in [n]_{-i}} \phi_{k}\left(h_{\pi(i)}^{(k-1)}\left(W,X\right),h_{\pi(j)}^{(k-1)}\left(W,X\right),w_{\pi(i),\pi(j)}\right)\\
	\label{a2}	  &= \sum_{j\in [n]_{-i}}\phi_{k}\left(h_{i}^{(k-1)}\left(\sigma_{\pi}(W),\lambda_{\pi}(X)\right),h_{j}^{(k-1)}\left(\sigma_{\pi}(W),\lambda_{\pi}(X)\right),w_{\pi(i),\pi(j)}\right),
	\end{align}
	where equation~\eqref{a1} is obtained by putting $j = \pi^{-1}(\ell)$ or equivalently $\ell=\pi(j)$ (note that any permutation is bijective by definition). Moreover, \cref{a2} holds due to the induction hypothesis for $k-1$. By letting $W'=\sigma_{\pi}(W)$ and $X'=\lambda_{\pi}(X)$, we note that $w_{i,j}'=w_{\pi(i),\pi(j)}$. Replacing these values in \eqref{a2} results in
	\begin{align}
	h_{\pi(i)}^{(k)}(W,X) &= \sum_{j\in [n]_{-i}}\phi_{k}\left(h_{i}^{(k-1)}\left(W',X'\right),h_{j}^{(k-1)}\left(W',X'\right),w_{i,j}'\right)= h_i^{(k)}\left(W',X'\right).
	\end{align}
	This concludes the induction and hence \eqref{ht} is proven. As a result, $H^{(k)}\in \mathcal{F}_{n,d}$.
\end{proof}

\section{Proof of \cref{theo: distinct}}\label{app: theo-distinct}
	\begin{proof}
	Suppose $F\in\mathcal{F}_{n,d}$ and let $F(W,X) = (f_1(W,X),\ldots,f_n(W,X))$. Consider a basis function $\mathcal{B}(W,X)=(\beta_1(W,X),\ldots,\beta_n(W,X))$ over $\G_{n,d}$ as defined in \cref{def: beta}.  
	Due to Theorem~\ref{theo: rho}, we can conclude that there exists a function $\rho$ such that for every $i\in [n]$, the following holds for all $G=([n],W,X)\in \tilde{\G}_{n,d}$: 
	\begin{equation}\label{aaa5}
	f_{i}(W,X) = \rho\left(\beta_{i}(W,X)\right).
	\end{equation}
	We introduce a GNN with three iterations, i.e., we introduce functions $\phi_1, \phi_2, \phi_3$ such that $h_{i}^{(3)} = \rho(\beta_{i}(W,X))$  for all $i\in [n]$ and as a result, $H^{(3)}(W,X) = F(W,X)$ for all $G=([n],W,X)\in \tilde{\G}_{n,d}$. First consider the definition of $\beta_{i}(W,X)$ and re-write \eqref{aaa5} as
	\begin{equation}
	f_{i}(W,X) = \rho\left(\beta_{i}(W,X)\right) = \rho\left(x_{i},\sum_{j\in [n]_{-i}}\psi_2\left(x_j,w_{i,j},\sum_{\ell\in [n]_{-j}}\psi_1\left(x_{\ell},w_{j,\ell}\right)\right)\right).
	\end{equation}
	
	Define the function $\phi_1$ of the GNN as
	\begin{align}\label{p1}
	\phi_1\left(h_j^{(0)},h_{\ell}^{(0)},w_{j,\ell}\right)=\left(\frac{1}{n-1}h_j^{(0)},\,\,\psi_1\left(h_{\ell}^{(0)},w_{j,\ell}\right)\right).
	\end{align}
	This leads to the following formula for all $j\in [n]$:
	\begin{align}
	h_j^{(1)} &= \sum_{\ell\in [n]_{-j}}\phi_1\left(h_j^{(0)},h_{\ell}^{(0)},w_{j,\ell}\right)\\
	&= \sum_{\ell\in [n]_{-j}}\left(\frac{1}{n-1}h_j^{(0)},\,\,\psi_1\left(h_{\ell}^{(0)},w_{j,\ell}\right)\right) = \left(x_j,\,\,\sum_{\ell\in [n]_{-j}}\psi_1\left(x_{\ell},w_{j,\ell}\right)\right).
	\end{align}
	Using the notation $[\cdot]_{a:b}$ introduced in \cref{app: notation}, define the function $\phi_2$ of the GNN as
	\begin{align}\label{p2}
	\phi_2\left(h_i^{(1)},h_j^{(1)},w_{i,j}\right)=\left(\frac{1}{n-1}\left[h_i^{(1)}\right]_{1:d}, \,\,\psi_2\left(\left[h_j^{(1)}\right]_{1:d},w_{i,j},\left[h_j^{(1)}\right]_{d+1:\text{end}}\right)\right).
	\end{align}
	Hence, 
	\begin{align}
	h_i^{(2)} &= \sum_{j\in [n]_{-i}}\phi_2\left(h_i^{(1)},h_j^{(1)},w_{i,j}\right) \\
	&= \sum_{j\in [n]_{-i}}\left(\frac{1}{n-1}\left[h_i^{(1)}\right]_{1:d}, \,\,\psi_2\left(\left[h_j^{(1)}\right]_{1:d},w_{i,j},\left[h_j^{(1)}\right]_{d+1:\text{end}}\right)\right)\\
	&= \left(x_i,\,\,\sum_{j\in [n]_{-i}}\psi_2\left(x_j,w_{i,j},\sum_{\ell\in [n]_{-j}}\psi_1\left(x_{\ell},w_{j,\ell}\right)\right)\right)=\beta_i(W,X).
	\end{align}
	Finally, define the function $\phi_3$ of the GNN as
	\begin{align}\label{p3}
	\phi_3\left(h_i^{(2)},h_j^{(2)},w_{i,j}\right)=\frac{1}{n-1}\rho\left(h_i^{(2)}\right),
	\end{align}
	which results in 
	\begin{align}
	\label{aaa4}h_i^{(3)}(W,X) &= \sum_{j\in [n]_{-i}}\phi_3\left(h_i^{(2)},h_j^{(2)},w_{i,j}\right)=  \sum_{j\in [n]_{-i}}\frac{1}{n-1}\rho\left(h_i^{(2)}\right) =\rho\left(h_i^{(2)}\right) =\rho\left(\beta_i(W,X)\right).
	\end{align}
	
	Putting \eqref{aaa5} and \eqref{aaa4} together, we conclude that $h_i^{(3)}(W,X)=f_i(W,X)$ for all $i\in [n]$ and all $G=([n],W,X)\in \tilde{\G}_{n,d}$ and consequently $H^{(3)}(W,X)=F(W,X)$.
\end{proof}

\section{Proof of Theorem~\ref{theo: arbitrary p.c}}\label{app: theo-arbitrary p.c}
\begin{proof}
	Define a graph function $F'=(f'_1,\ldots,f'_n)$ over $\G_{d+d_0}$ such that $F'(W,X^{\ast})=F(W,X)$, i.e., $f'_i(W,X^{\ast})=f_i(W,X)$ for all $i\in [n]$. Note that $F'$ ignores the last $d_0$ coordinates of node features and returns $F$. Hence, for all $i\in [n]$, we have 
	\begin{align}
		\label{z1} f'_i(\sigma_{\pi}(W),\lambda_{\pi}(X^{\ast}))=f_i(\sigma_{\pi}(W),\lambda_{\pi}(X)) = f_{\pi(i)}(W,X) = f'_{\pi(i)}(W,X^{\ast}),
	\end{align}  
	where the first and the third equality in \eqref{z1} hold due to the definition of $F'$ and the second equality holds due to the permutation compatibility of $F$. \cref{z1} then implies that $F'$ is permutation compatible. Therefore, due to \cref{theo: distinct}, there exists $H^{(k)}$ such that $H^{(k)}(W,Z)=F'(W,Z)$ for all $G=([n],W,Z)\in \tilde{\G}_{n,d+d_0}$. Note that for any $G=([n],W,X)\in \G_{n,d}$, we have $G=([n],W,X^{\ast})\in \tilde{\G}_{n,d+d_0}$. As a result, $H^{(k)}(W,X^{\ast})=F'(W,X^{\ast}) =F(W,X)$ for all $G=([n],W,X)\in \G_{n,d}$.
\end{proof}

\section{Proof of Proposition~\ref{prop: acf-equivalent}}\label{app: prop-acf-equivalent}

\begin{lemma}\label{lem: intermediate2}
	Consider a graph $G=([n],W,X)$ and suppose $\pi,\pi_1,\pi_2\in S_n$. Then $\lambda_{\pi_2 \circ \pi_1}(X) = \lambda_{\pi_1} \circ \lambda_{\pi_2} (X)$ and $\sigma_{\pi_2 \circ \pi_1}(W) = \sigma_{\pi_1} \circ \sigma_{\pi_2} (W)$ (note how the order of the composition changes). In particular, $\lambda_{\pi^{-1}}(X) = \lambda_{\pi}^{-1} (X)$ and $\sigma_{\pi^{-1}}(W) = \sigma_{\pi}^{-1} (W)$. 
\end{lemma}
\begin{proof}[{\bf Proof of \cref{lem: intermediate2}}]
	Note that $\lambda_{\pi}(X)$ is in general a sequence of objects. To show that two sequences are equal, it suffices to show that their corresponding elements are equal. To this end, fix $i\in [n]$ and note that $\left[\lambda_{\pi_2 \circ \pi_1}(X)\right]_i = x_{\pi_2 \circ \pi_1(i)}$. Moreover, let $y_i = \left[\lambda_{\pi_2}(X)\right]_i=x_{\pi_2(i)}$ and $Y=(y_1,\ldots,y_n)=\lambda_{\pi_2}(X)$. Therefore, 
	\begin{align}
	    \left[\lambda_{\pi_1} \circ \lambda_{\pi_2} (X)\right]_i = \left[\lambda_{\pi_1} \left( Y\right)\right]_i =y_{\pi_1(i)} = x_{\pi_2 \circ \pi_1(i)}.
	\end{align}
	Hence, we showed that $\left[\lambda_{\pi_2 \circ \pi_1}(X)\right]_i = x_{\pi_2 \circ \pi_1(i)} = \left[\lambda_{\pi_1} \circ \lambda_{\pi_2} (X)\right]_i$ for all $i\in [n]$, which leads to $\lambda_{\pi_2 \circ \pi_1}(X) = \lambda_{\pi_1} \circ \lambda_{\pi_2} (X)$.
	
	The argument for $\sigma_{\pi}(W)$ is similar. For fixed and distinct $i,j\in [n]$, note that $\left[\sigma_{\pi_2 \circ \pi_1}(W)\right]_{i,j}= w_{\pi_2 \circ \pi_1(i), \pi_2 \circ \pi_1(j)}$. Define $Z=\sigma_{\pi_2} (W)$ whose $(i,j)$-th element is $z_{i,j}$. This means $z_{i,j} = w_{\pi_2(i),\pi_2(j)}$. Now we have 
	\begin{align}
	    \left[\sigma_{\pi_1} \circ \sigma_{\pi_2} (W)\right]_{i,j} = \left[\sigma_{\pi_1} \left(Z\right)\right]_{i,j} =  z_{\pi_1(i),\pi_1(j)} = w_{\pi_2 \circ \pi_1(i), \pi_2 \circ \pi_1(j)}.
	\end{align}
	Hence, we showed that $\left[\sigma_{\pi_2 \circ \pi_1}(W)\right]_{i,j}= w_{\pi_2 \circ \pi_1(i), \pi_2 \circ \pi_1(j)}= \left[\sigma_{\pi_1} \circ \sigma_{\pi_2} (W)\right]_{i,j}$ for all distinct $i,j\in[n]$, which leads to $\sigma_{\pi_2 \circ \pi_1}(W) = \sigma_{\pi_1} \circ \sigma_{\pi_2} (W)$. 
	
	Note that $\lambda_{\pi}(\cdot)$ and $\sigma_{\pi}(\cdot)$ are permutations and thus their inverse exist and right and left inverses coincide. Based on the first part of the statement, we can write $X = \lambda_{\operatorname{id}}(X)=\lambda_{\pi \circ \pi^{-1}}(X) = \lambda_{\pi^{-1}} \circ \lambda_{\pi}(X)$, where $\operatorname{id}$ is the identity permutation. Hence, $\lambda_{\pi^{-1}}(X) = \lambda_{\pi}^{-1} (X)$. Similarly for $\sigma_{\pi}(W)$, we have $X = \sigma_{\operatorname{id}}(X)=\sigma_{\pi \circ \pi^{-1}}(X) = \sigma_{\pi^{-1}} \circ \sigma_{\pi}(X)$. Therefore, $\sigma_{\pi^{-1}}(W) = \sigma_{\pi}^{-1} (W)$. 
\end{proof}

\begin{lemma}\label{prop: node-aif-basis}
	For a graph $G=([n],W,X)$ and function $f(W,X)$, the following statements are equivalent:
	\begin{align}
		\label{aa1}\text{(i)}\quad&f(\sigma_{\pi_{r,s}}(W),\lambda_{\pi_{r,s}}(X))=f(W,X)\quad \forall r,s\in [n]_{-i}.\\
		\label{aa2}\text{(ii)}\quad&f(\sigma_{\pi}(W),\lambda_{\pi}(X))=f(W,X)\quad \forall \pi \in \nabla_i.
	\end{align}
\end{lemma}
\begin{proof}[{\bf Proof of \cref{prop: node-aif-basis}}]
	Since for $r,s\in[n]_{-i}$, we have $\pi_{r,s}\in\nabla_i$, we conclude that \eqref{aa1} follows from \eqref{aa2}. Therefore, it suffices to show that \eqref{aa2} follows from \eqref{aa1}. To this end, suppose $\pi\in\nabla_i$. Then $\pi$ fixes $i$, i.e., $\pi(i)=i$ and induces a permutation over $[n]_{-i}$. Call this induced permutation $\tilde{\pi}$. It is known that any permutation can be written as a composition of transpositions, i.e., swapping permutations (see \cite{bhattacharya_jain_nagpaul_1994}). This means that there exists a sequence of swappings $r_1 \leftrightarrow s_1, \ldots, r_k \leftrightarrow s_k$ over $[n]_{-i}$ whose composition is $\tilde{\pi}$. Note that when the composition of $r_1 \leftrightarrow s_1, \ldots, r_k \leftrightarrow s_k$ is considered over $[n]$ instead of $[n]_{-i}$, it equals to $\pi$. To summarize this argument, there exist $(r_1,s_1),\ldots,(r_m,s_m)$, where $s_{\ell},r_{\ell}\in [n]_{-i}$ and $r_{\ell}\neq s_{\ell}$ for all $\ell\in[m]$ such that $\pi = \pi_{r_1,s_1}\circ \pi_{r_2,s_2}\circ \ldots \circ \pi_{r_m,s_m}$. Having this, for $\pi\in \nabla_i$, we can write
	\begin{align}
	f\left(\sigma_{\pi}(W),\lambda_{\pi}(X)\right)&= f\left(\sigma_{\pi_{r_1,s_1}\circ \ldots \circ \pi_{r_m,s_m}}(W),\lambda_{\pi_{r_1,s_1}\circ  \ldots \circ \pi_{r_m,s_m}}(X)\right)\\
	\label{a26}&=f\left(\sigma_{\pi_{r_m,s_m}}\circ  \ldots \circ \sigma_{\pi_{r_1,s_1}}(W),\lambda_{\pi_{r_m,s_m}}\circ  \ldots \circ \lambda_{\pi_{r_1,s_1}}(X)\right)\\
	\label{a27}&=f(W,X),
	\end{align}
	where \eqref{a26} holds due to applying Lemma~\ref{lem: intermediate2}, $m$ times and \eqref{a27} follows from applying \eqref{aa1}, $m$ times. 
\end{proof}

\begin{proof}[\bf Proof of Proposition~\ref{prop: acf-equivalent}]
	The conditions \eqref{a19} and \eqref{a29} are particular cases of \eqref{a23}. Therefore, they hold trivially if $F\in \mathcal{F}_{n,d}$. Now suppose both of the conditions \eqref{a19} and \eqref{a29} hold. We want to show \eqref{a23} for all $\pi\in S_n$. For a given $\pi\in S_n$ and $r\in [n]$, we want to show that 
	\begin{align}\label{a28}
	f_{\pi{(r)}}(W,X) = f_r(\sigma_{\pi}(W),\lambda_{\pi}(X)).
	\end{align}
	First, note that due to Lemma~\ref{prop: node-aif-basis}, condition \eqref{a19} is equivalent to 
	\begin{align}
	\label{aa3}f_{i_0}(\sigma_{\pi}(W),\lambda_{\pi}(X))=f_{i_0}(W,X)\quad \forall \pi \in \nabla_{i_0}.
	\end{align}
	Having \eqref{aa3} as an equivalent of condition \eqref{a19}, we proceed as follows. Given $r\in [n]$, let $s=\pi(r)$ and consider the following cases. In each case, we show that \eqref{a28} holds for the specified $r$ and $s$.
	\begin{itemize}
		\item Case $r=s=i_0$. In this case $\pi(i_0)=i_0$ and thus $\pi\in \nabla_{i_0}$. Therefore, \eqref{aa3} implies that
		\begin{align}
		f_{i_0}(W,X) = f_{i_0}(\sigma_{\pi}(W),\lambda_{\pi}(X)).
		\end{align}
		
		\item Case $r=i_0$ and $s\neq i_0$. We have
		\begin{align}
		\label{a30}f_{s}(W,X) &= f_{i_0}(\sigma_{\pi_{s,i_0}}(W),\lambda_{\pi_{s,i_0}}(X))\\
		\label{a31}&=f_{i_0}(\sigma_{\pi \circ \pi ^{-1}\circ\pi_{s,i_0}}(W),\lambda_{\pi \circ \pi ^{-1}\circ\pi_{s,i_0}}(X))\\
		\label{a32}&=f_{i_0}(\sigma_{\pi ^{-1}\circ\pi_{s,i_0}}  \circ  \sigma_{\pi}(W),\lambda_{\pi ^{-1}\circ\pi_{s,i_0}}  \circ  \lambda_{\pi}(X)),
		\end{align}
		where the equality \eqref{a30} follows from \eqref{a29}, the equality \eqref{a31} holds because $\pi \circ \pi ^{-1}\circ\pi_{s,i_0}=\pi_{s,i_0}$, and the equality \eqref{a32} is a results of Lemma~\ref{lem: intermediate2}. Note that $\pi ^{-1}\circ\pi_{s,i_0} \in \nabla_{i_0}$ because $\pi ^{-1}\circ\pi_{s,i_0} (i_0) = \pi ^{-1}(s) = i_0$ and thus, due to \eqref{aa3}, we have
		\begin{align}
		\label{a33p}f_{i_0}(\sigma_{\pi ^{-1}\circ\pi_{s,i_0}}  (W'),\lambda_{\pi ^{-1}\circ\pi_{s,i_0}}  (X')) = f_{i_0}(W',X').
		\end{align}
		Replace $W'=\sigma_{\pi}(W)$ and $X'=\lambda_{\pi}(X)$ in \eqref{a33p} to get
		\begin{align}
		\label{a33}f_{i_0}(\sigma_{\pi ^{-1}\circ\pi_{s,i_0}}  \circ  \sigma_{\pi}(W),\lambda_{\pi ^{-1}\circ\pi_{s,i_0}}  \circ  \lambda_{\pi}(X)) = f_{i_0}(\sigma_{\pi}(W),\lambda_{\pi}(X)).
		\end{align}
		Hence, \eqref{a32} and \eqref{a33} together lead to
		\begin{align}
		f_{s}(W,X) =f_{i_0}(\sigma_{\pi}(W),\lambda_{\pi}(X)).
		\end{align}
		
		\item Case $r\neq i_0$ and $s=i_0$. We want to show that $f_{i_0}(W,X) = f_r(\sigma_{\pi}(W),\lambda_{\pi}(X))$. This is equivalent to $f_{r}(W,X) = f_{i_0}(\sigma_{\pi}^{-1}(W),\lambda_{\pi}^{-1}(X))=f_{i_0}(\sigma_{\pi^{-1}}(W),\lambda_{\pi^{-1}}(X))$, which follows from the previous case ($r=i_0$ and $s\neq i_0$) by replacing $\pi$ with $\pi^{-1}$.
		
		\item Case $s\neq i_0$ and $r\neq i_0$. 
		\begin{align}
		\label{a34}f_{r}(\sigma_{\pi}(W),\lambda_{\pi}(X)) &= f_{i_0}(\sigma_{\pi_{r,i_0}}\circ \sigma_{\pi}(W),\lambda_{\pi_{ri_0}} \circ \lambda_{\pi}(X))\\
		\label{a35}&= f_{i_0}(\sigma_{\pi \circ \pi_{r,i_0}}(W),\lambda_{\pi \circ \pi_{r,i_0}}(X))\\
		\label{a36}&= f_{i_0}(\sigma_{\pi_{s,i_0}\circ \pi_{s,i_0}\circ\pi \circ \pi_{r,i_0}}(W),\lambda_{\pi_{s,i_0}\circ \pi_{s,i_0}\circ\pi \circ \pi_{r,i_0}}(X))\\
		\label{a37}&= f_{i_0}(\sigma_{\pi_{s,i_0}\circ\pi \circ \pi_{r,i_0}} \circ \sigma_{\pi_{s,i_0}}(W),\lambda_{\pi_{s,i_0}\circ\pi \circ \pi_{r,i_0}} \circ \lambda_{\pi_{s,i_0}}(X)),
		\end{align}
		where the equality \eqref{a34} follows from \eqref{a29}, the equality \eqref{a36} holds because $\pi_{s,i_0}^{-1}= \pi_{s,i_0}$ and thus $\pi_{s,i_0}\circ \pi_{s,i_0}\circ\pi \circ \pi_{r,i_0}=\pi \circ \pi_{r,i_0}$, and the equality \eqref{a35} and \eqref{a37} are results of Lemma~\ref{lem: intermediate2}. Note that $\pi_{s,i_0}\circ\pi \circ \pi_{r,i_0} \in \nabla_{i_0}$ because $\pi_{s,i_0}\circ\pi \circ \pi_{r,i_0}(i_0) = \pi_{s,i_0}\circ\pi (r) = \pi_{s,i_0}(s)=i_0$ and thus due to \eqref{aa3}, we have
		\begin{align}
		\label{a38p}f_{i_0}(\sigma_{\pi_{s,i_0}\circ\pi \circ \pi_{r,i_0}} (W') ,\lambda_{\pi_{s,i_0}\circ\pi \circ \pi_{r,i_0}} (X') ) = f_{i_0}(W',X').
		\end{align}
		Replace $W'=\sigma_{\pi_{s,i_0}}(W)$ and $X'=\lambda_{\pi_{s,i_0}}(X)$ in \eqref{a38p} to get
		\begin{align}
		\label{a38}f_{i_0}(\sigma_{\pi_{s,i_0}\circ\pi \circ \pi_{r,i_0}} \circ \sigma_{\pi_{s,i_0}}(W),\lambda_{\pi_{s,i_0}\circ\pi \circ \pi_{r,i_0}} \circ \lambda_{\pi_{s,i_0}}(X)) = f_{i_0}(\sigma_{\pi_{s,i_0}}(W),\lambda_{\pi_{s,i_0}}(X)).
		\end{align}
		Moreover, for the right-hand side of \eqref{a38}, we use \eqref{a29} to write
		\begin{align}
		\label{a39}f_{i_0}(\sigma_{\pi_{s,i_0}}(W),\lambda_{\pi_{s,i_0}}(X)) = f_s(W,X).
		\end{align}
		Putting together \eqref{a37}, \eqref{a38}, and \eqref{a39}, we have
		\begin{align}
		f_{r}(\sigma_{\pi}(W),\lambda_{\pi}(X)) = f_s(W,X).
		\end{align}
	\end{itemize}
	Having that \eqref{a28} holds in all cases, we conclude that $F\in \mathcal{F}_{n,d}$.
\end{proof}

\section{Formal statement and proof of \cref{corr-newp}}\label{app: pc-ex}
In this section, we decompose \cref{corr-newp} into two formal corollaries. The part (i) of \cref{corr-newp} can be formally stated as follows.
\begin{corollary}[Formal]\label{corr-new} 
	Suppose $F(W,X)$ is a function over $\G_{n,d}$ that ignores $W$, i.e., $F(X) = (f_1(X),\ldots,f_n(X))$. Then $F\in \mathcal{F}_{n,d}$ if and only if for some $i_0\in [n]$: $f_{i_0}(X) = f_{i_0}(\lambda_{\pi_{r,s}}(X))$ for all $r,s\in [n]_{-i_0}$ and  $f_j(X) = f_{i_0}(\lambda_{\pi_{i_0,j}}(X))$ for all $j\in [n]_{-i_0}$.
\end{corollary} 
\cref{corr-new} directly follows from  \cref{prop: acf-equivalent}. A re-statement of \cref{corr-new} in the following form better presents the result. We start with the following definition.
\begin{definition}
    Consider a function $f$, that accepts $n$ inputs from $\R^d$. Then $f$ is called  {\it quasi permutation invariant} if for all $x,y_1,\ldots,y_{n-1}\in \R^d$ and all $\pi\in S_{n-1}$, we have
    \begin{align}
        f\left(x,y_{\pi(1)},\ldots,y_{\pi(n-1)}\right) = f\left(x,y_1,\ldots,y_{n-1}\right).
    \end{align}
\end{definition}
\begin{corollary}\label{corr-new2} 
	Suppose $F(W,X)$ is a function over $\G_{n,d}$ that ignores $W$, i.e., $F(X) = (f_1(X),\ldots,f_n(X))$. Further, let $X_{-i}$ be the sequence $(x_1,\ldots,x_n)$ from which $x_i$ is removed. Then $F\in \mathcal{F}_{n,d}$ if and only if $f_i(X)=f(x_i,X_{-i})$ for all $i\in[n]$, where $f$ is some quasi-permutation-invariant function.
\end{corollary}

\cref{corr-new2} is just a re-statement and an immediate result of \cref{corr-new}.  As a side result of \cref{corr-new2}, having a fully connected unweighted graph with features $x_1,\ldots,x_n$, any function $f(x_1,\ldots,x_n)$ that remains invariant under any permutation of $x_1,\ldots,x_n$ can be generated by a GNN. This particular case was also proven in \cite{xu2020what} based on the analytical results from \cite{zaheer2017deepset}.

Part (ii) of \cref{corr-newp} can be formally stated as follows.

\begin{corollary}\label{corr: per-inv}
	Consider a function $F = (f_1,\ldots,f_n)$ over $\G_{n,d}$ and suppose $F$ assigns the same value to all nodes, i.e., $f_i=f$ for all $i\in[n]$. Moreover,  assume $f$ is an isomorphism-invariant function, i.e.,  $f(W,X)=f(\sigma_{\pi}(W),\lambda_{\pi}(X))$ for all $\pi\in S_n$ and all $G=([n],W,X)$ in $\G_{n,d}$. Then $F\in\F_{n,d}$. 
\end{corollary}
\cref{{corr: per-inv}} directly follows from the definition of permutation compatibility in \cref{a23} by replacing $f_i=f$ for all $i\in [n]$.

\section{Proof of Proposition~\ref{prop: cef}} \label{app: prop-cef}

\begin{lemma}\label{lem: intermediate1}
	If for $v_1,\ldots,v_n, v_1',\ldots,v_n'\in \C$, 
	\begin{align}\label{a13}
	\sum_{i=1}^n v_i^r = \sum_{i=1}^n v_i^{'r}\quad \forall r\in\{1,\ldots,n\},
	\end{align}
	then $\#\{v_1,\ldots,v_n\}=\#\{v_1',\ldots,v_n'\}$.
\end{lemma}
\begin{proof}[{\bf Proof of \cref{lem: intermediate1}}]
	We use the well-known {\it power sum symmetric polynomials} $p_r(\cdot)$ as well as the {\it elementary symmetric polynomials} $e_r(\cdot)$ over $n$ variables which are defined as follows:
	\begin{align}
	p_r\left(v_1,\ldots,v_n\right) = \sum_{i=1}^n v_i^r, \quad\quad	e_r\left(v_1,\ldots,v_n\right) = \left\{
	\begin{array}{lr}
	1, & r=0,\\
	\sum_{1\leq i_1<i_2<\ldots<i_r\leq n}v_{i_1}\ldots v_{i_r}, & 1\leq r\leq n.
	\end{array}\right.
	\end{align}  
	Using this notation, we have $p_r\left(v_1,\ldots,v_n\right) = p_r\left(v_1',\ldots,v_n'\right)$ for all $r\in\{1,\ldots,n\}$. First, we show that $e_r\left(v_1,\ldots,v_n\right) = e_r\left(v_1',\ldots,v_n'\right)$ also holds for all $r\in\{1,\ldots,n\}$. To this end, we use induction on $r$. For $r=1$, $e_1$ coincides with $p_1$, i.e., 
	\begin{align}
	e_1\left(v_1,\ldots,v_n\right)= \sum_{i=1}^n v_i =p_1\left(v_1,\ldots,v_n\right).
	\end{align}
	Hence,  
	\begin{align}
	e_1\left(v_1,\ldots,v_n\right)= p_1\left(v_1,\ldots,v_n\right)=p_1\left(v_1',\ldots,v_n'\right)=e_1\left(v_1',\ldots,v_n'\right).
	\end{align}
	For $r>1$, having the result for all values less than $r$, we need to prove it for $r$. From Newton's identity for elementary symmetric polynomial, we have
	\begin{align}
	\label{b1}e_r\left(v_1,\ldots,v_n\right) = \frac{1}{r}\,\sum_{i=1}^{r}(-1)^{i-1} e_{r-i}\left(v_{1}, \ldots, v_{n}\right) p_{i}\left(v_{1}, \ldots, v_{n}\right).
	\end{align}
	The identity in \eqref{b1} together with the induction hypothesis and the assumption on $p_r$, results in
	\begin{align}
	e_r\left(v_1,\ldots,v_n\right) &= \frac{1}{r}\,\sum_{i=1}^{r}(-1)^{i-1} e_{r-i}\left(v_{1}, \ldots, v_{n}\right) p_{i}\left(v_{1}, \ldots, v_{n}\right)\\
	&=\frac{1}{r}\,\sum_{i=1}^{r}(-1)^{i-1} e_{r-i}\left(v_{1}', \ldots, v_{n}'\right) p_{i}\left(v_{1}', \ldots, v_{n}'\right)\\
	&=e_r\left(v_1',\ldots,v_n'\right).
	\end{align}
	Hence, for all $r\in\{0,\ldots,n\}$, we have
	\begin{align}\label{b2}
	e_r\left(v_1,\ldots,v_n\right) =e_r\left(v_1',\ldots,v_n'\right).
	\end{align}
	Now consider the polynomial $\prod_{i=1}^{n}(x-v_{i})$ over the complex variable $x$. We can decompose this polynomial and get the identity $\prod_{i=1}^{n}\left(x-v_{i}\right)=\sum_{r=0}^{n}(-1)^{r} e_{r}\left(v_1,\ldots,v_n\right) x^{n-r}$. Replacing \eqref{b2} in this identity, leads to
	\begin{align}
	\prod_{i=1}^{n}\left(x-v_{i}\right)=\sum_{r=0}^{n}(-1)^{r} e_{r}\left(v_1,\ldots,v_n\right) x^{n-r}=\sum_{r=0}^{n}(-1)^{r} e_{r}\left(v_1',\ldots,v_n'\right) x^{n-r} = \prod_{i=1}^{n}\left(x-v_{i}'\right).
	\end{align}
	Hence, $\#\{v_1,\ldots,v_n\}=\#\{v_1',\ldots,v_n'\}$.
\end{proof}
\begin{proof}[{\bf Proof of Proposition~\ref{prop: cef}}]
	\begin{enumerate}[(i)]
		\item We need to show that
		\begin{align}
		\sum_{i=1}^n \psi\left(v_i\right)=\sum_{i=1}^n \psi\left(v_i'\right)\quad\Longleftrightarrow\quad \#\{v_1,\ldots,v_n\}=\#\{v_1',\ldots,v_n'\}.
		\end{align} 
		Note that $\sum_{i=1}^n \psi(v_i)=\sum_{i=1}^n \psi(v_i')$ is equivalent to \eqref{a13} and thus the statement follows from Lemma~\ref{lem: intermediate1}. 
		\item If $\#\{v_1,\ldots,v_n\} = \#\{v_1',\ldots,v_n'\}$, then \eqref{a4} holds trivially. Therefore, suppose \eqref{a4} holds for $\psi$, i.e., 
		\begin{align}\label{a17}
		\sum_{i=1}^n \psi\left(v_i\right) = \sum_{i=1}^n \psi\left(v_i'\right).
		\end{align} 
		We claim that $\#\{v_1,\ldots,v_n\} = \#\{v_1',\ldots,v_n'\}$. Without loss of generality, consider the function $\psi$ in its tensor-form rather than its linearized vector-form. Now for $r<s$, taking the $(\ell,r,s)$-th and $(\ell,s,r)$-th coordinate of both sides of the equation \eqref{a17} leads to the following equations:
		\begin{align}
		\label{z3}\sum_{i=1}^n\operatorname{Re}\left( \left(\left[v_i\right]_r+\left[v_i\right]_s\sqrt{-1}\right)^{\ell}\right) &= \sum_{i=1}^n\operatorname{Re}\left( \left(\left[v_i'\right]_r+\left[v_i'\right]_s\sqrt{-1}\right)^{\ell}\right)\\
		\label{z4}\sum_{i=1}^n\operatorname{Im}\left( \left(\left[v_i\right]_r+\left[v_i\right]_s\sqrt{-1}\right)^{\ell}\right) &= \sum_{i=1}^n\operatorname{Im}\left( \left(\left[v_i'\right]_r+\left[v_i'\right]_s\sqrt{-1}\right)^{\ell}\right).
		\end{align}
		Hence, for every $r,s\in [m]$ with $r<s$, we have
		\begin{align}\label{a16}
		\sum_{i=1}^n\left(\left[v_i\right]_r+\left[v_i\right]_s\sqrt{-1}\right)^{\ell} &= \sum_{i=1}^n\left(\left[v_i'\right]_r+\left[v_i'\right]_s\sqrt{-1}\right)^{\ell} \quad \forall \ell\in [n].
		\end{align}
		Using Lemma~\ref{lem: intermediate1}, for every $r,s\in [m]$, the equation \eqref{a16} leads to 
		\begin{align}
		\#\left\{\left[v_i\right]_r+\left[v_i\right]_s\sqrt{-1} \mid i\in[n]\right\}=	\#\left\{\left[v_i'\right]_r+\left[v_i'\right]_s\sqrt{-1}\mid i\in[n]\right\},
		\end{align}
		which is equivalent to
		\begin{align}\label{a18}
		\#\left\{(\left[v_i\right]_r,\left[v_i\right]_s) \mid i\in[n]\right\}=	\#\left\{(\left[v_i'\right]_r,\left[v_i'\right]_s) \mid i\in[n]\right\}.
		\end{align}
		Since \eqref{a18} holds for every two coordinates $r,s\in [m]$ with $r<s$, we conclude that $\#\{v_1,\ldots,v_n\} = \#\{v'_1,\ldots,v'_n\}$.
	\end{enumerate}
\end{proof}

\section{Proof of Proposition~\ref{prop: beta}}\label{app: prop-beta}
\begin{definition}\label{def: delta}
	Consider the graph $G=([n],W,X)$. For each $s\in [n]$, let $W_s=\#\{(x_r,w_{s,r})\mid r \in [n]_{-s}\}$. Then for a given $i\in[n]$, define $\Delta_i(W,X)$ as 
	\begin{align}
	\Delta_i(W,X) = \left(x_i,\#\left\{(x_j,w_{i,j},W_j) \mid j\in [n]_{-i}\right\}\right).
	\end{align}
\end{definition}

\begin{lemma}\label{lem: invariant-object}
	Consider the graphs $G=([n],W,X)$ and $G'=([n],W',X')$ in $\tilde{\G}_{n,d}$. Then $\Delta_i(W,X) = \Delta_i(W',X')$ if and only if there exists $\pi\in \nabla_i$ such that $W' = \sigma_{\pi}(W)$ and $X' = \lambda_{\pi}(X)$.
\end{lemma}
\begin{proof}[{\bf Proof of \cref{lem: invariant-object}}]
	Denote the objects in the multiset by $\alpha_j = (x_j,w_{i,j},W_j)$ and $\alpha_j' = (x_j',w_{i,j}',W_j')$ for every $j\in [n]_{-i}$. To prove the sufficiency part, note that if $W' = \sigma_{\pi}(W)$ and $X' = \lambda_{\pi}(X)$ for some $\pi\in \nabla_i$, then we have $x_{i}'=x_i$ (because $\pi(i)=i$) as well as $x_{j}'=x_{\pi(j)}$ and $w_{i,j}'=w_{i,\pi(j)}$ for $j \neq i$. Letting $W_s'=\#\{(x_r',w_{s,r}')\mid r \in [n]_{-s}\}$, it is straightforward to see that $W_j'=W_{\pi(j)}$. So far, we have shown that for $j\neq i$:
	\begin{align}
	\alpha_j' = (x_j',w_{i,j}',W_j')=(x_{\pi(j)},w_{i,\pi(j)},W_{\pi(j)})=\alpha_{\pi(j)}.
	\end{align}
	This leads to $\#\{\alpha_j \mid j\in [n]_{-i}\} = \#\{\alpha_j' \mid j\in [n]_{-i}\}$ and thus $\Delta_i(W,X) = \Delta_i(W',X')$.
	
	To prove the necessity part, note that if $\Delta_i(W,X)=\Delta_i(W',X')$, then  both of the following conditions hold
	\begin{align}
	\label{a10}&x_i = x_i',\\
	\label{a11}&\#\left\{(x_j,w_{i,j},W_j) \mid j\in [n]_{-i}\right\}=\#\left\{(x_j',w_{i,j}',W_j') \mid j\in [n]_{-i}\right\}.
	\end{align}
	From \eqref{a11}, we conclude that there exists a permutation $\tilde{\pi}$ over $[n]_{-i}$ such that $\alpha_j'=\alpha_{\tilde{\pi}(j)}$. The permutation $\tilde{\pi}$ can be extended to a permutation $\pi$ over $[n]$ by defining $\pi(i)=i$ and $\pi(j)=\tilde{\pi}(j)$ for $j\neq i$. Note that $\pi\in \nabla_i$. We now claim that $W' = \sigma_{\pi}(W)$ and $X' = \lambda_{\pi}(X)$. Note that $X' = \lambda_{\pi}(X)$ trivially holds because for $j\neq i$,  we have $\alpha_j'=\alpha_{\tilde{\pi}(j)}=\alpha_{\pi(j)}$ which leads to $x_j' = x_{\pi(j)}$ and we already have $x_i' = x_i = x_{\pi(i)}$ from \eqref{a10}. Hence, $X' = \lambda_{\pi}(X)$. 
	
	To show that $W' = \sigma_{\pi}(W)$, it suffices to show that for all $r,s\in[n]$ with $r\neq s$, we have $w_{r,s}'=w_{\pi(r),\pi(s)}$. Considering this equality for $r=i$, note that $w_{i,j}'=w_{i,\tilde{\pi}(j)} = w_{\pi(i),\pi(j)}$ holds for $j\neq i$ because  $\alpha_j'=\alpha_{\tilde{\pi}(j)}$. Thus, it remains to prove $w_{r,s}'=w_{\pi(r),\pi(s)}$ when $r$ and $s$ are not equal to $i$. To this end, we proceed as follows. From $\alpha_j'=\alpha_{\tilde{\pi}(j)}=\alpha_{{\pi}(j)}$, we conclude that
	\begin{align}
	\label{a12}\quad W_j'=W_{{\pi}(j)}\quad\text{ for } j\in [n]_{-i}.
	\end{align} 
	This means that for $j\in [n]_{-i}$
	\begin{align}\label{ad}
	\left\{(x_{\ell}',w_{j,\ell}')\mid \ell \in [n]_{-j}\right\} = 	\left\{(x_{k},w_{\pi(j),k})\mid k \in [n]_{-\pi(j)}\right\}.
	\end{align}
	Note that we already know that $x_{r}' = x_{\pi(r)}$ holds for all $r\in [n]$ and since the elements $x_1',\ldots,x_n'$ are distinct as well as $x_1,\ldots,x_n$, this $\pi$ is unique. Having this, \cref{ad} implies that $w_{j,\ell}'= w_{\pi(j),\pi(\ell)}$ holds for all $\ell\in [n]_{-j}$. Hence, we showed that $w_{j,\ell}'=w_{\pi(j),\pi(\ell)}$ holds for all $j,\ell\in [n]_{-i}$. The case where either $\ell$ or $j$ is equal to $i$ was proven above.  As a result,  we have $W' = \sigma_{\pi}(W)$, which completes the proof.
\end{proof}

\begin{lemma}\label{lem: gi-Deltai}
	Suppose $G=([n],W,X)$ and $G'=([n],W',X')$ are two graphs. Then $\beta_i(W,X)=\beta_i(W',X')$ if and only if $\Delta_i(W,X)=\Delta_i(W',X')$.
\end{lemma}
\begin{proof}[{\bf Proof of \cref{lem: gi-Deltai}}]
	Note that $\beta_i(W,X)=\beta_i(W',X')$ holds if and only if both of the following conditions hold
	\begin{align}
	\label{a9}&x_i = x_i',\\
	\label{a7}&\sum_{j\in [n]_{-i}}\psi_2\left(x_j,w_{i,j},\sum_{\ell\in [n]_{-j}}\psi_1\left(x_{\ell}, w_{j,\ell}\right)\right)=\sum_{j\in [n]_{-i}}\psi_2\left(x_j',w_{i,j}',\sum_{\ell\in [n]_{-j}}\psi_1\left(x_{\ell}', w_{j,\ell}'\right)\right).
	\end{align}
	Since $\psi_2$ is an MEF (defined in Definition~\ref{def: cef}), The equation \eqref{a7} holds if and only if
	\begin{align}
	\label{a8}\#\{(x_j,w_{i,j},\sum_{\ell\in [n]_{-j}}\psi_1\left(x_{\ell}, w_{j,\ell}\right)) \mid j\in [n]_{-i}\} = \#\{(x_j',w_{i,j}',\sum_{\ell\in [n]_{-j}}\psi_1\left(x_{\ell}', w_{j,\ell}'\right)) \mid j\in [n]_{-i}\}.
	\end{align}
	Due to the fact that $\psi_1$ is a MEF, we know that 
	\begin{align}
	\sum_{\ell\in [n]_{-j}}\psi_1\left(x_{\ell}, w_{j,\ell}\right) = \sum_{\ell\in [n]_{-j}}\psi_1\left(x_{\ell}',w_{j,\ell}'\right)\quad \Longleftrightarrow \quad W_j = W_j'.
	\end{align}
	Therefore, \eqref{a8} holds if and only if 
	\begin{align}\label{aa4}
	\#\left\{\left(x_j,w_{i,j},W_j\right) \mid j\in [n]_{-i}\right\} = \#\left\{\left(x_j',w_{i,j}',W_j'\right) \mid j\in [n]_{-i}\right\}.
	\end{align}
	Knowing that \eqref{a7} is equivalent to \eqref{aa4}, $\beta_i(W,X)=\beta_i(W',X')$ holds if and only if 
	\begin{align}
	\label{aa9}&x_i = x_i',\\
	\label{aa7}&\#\left\{\left(x_j,w_{i,j},W_j\right) \mid j\in [n]_{-i}\right\} = \#\left\{\left(x_j',w_{i,j}',W_j'\right) \mid j\in [n]_{-i}\right\}.
	\end{align}
	Finally, \eqref{aa9} and \eqref{aa7} hold if and only if $\Delta_i(W,X)=\Delta_i(W',X')$, due to \cref{def: delta}. Hence, we showed that $\beta_i(W,X)=\beta_i(W',X')$ if and only if $\Delta_i(W,X)=\Delta_i(W',X')$.
\end{proof}

\begin{proof}[{\bf Proof of Proposition~\ref{prop: beta}}]
	We prove the additional property first. Note that if $\beta_i(W,X)=\beta_i(W',X')$ holds for two graphs in $\tilde{\G}_{n,d}$, then due to Lemma~\ref{lem: gi-Deltai}, $\Delta_i(W,X)=\Delta_i(W',X')$ and thus from Lemma~\ref{lem: invariant-object}, there exists $\pi\in \nabla_i$ such that $W' = \sigma_{\pi}(W)$ and $X' = \lambda_{\pi}(X)$.
	
	It remains to prove that $\mathcal{B}\in \F_{n,d}$. To this end, we use Proposition~\ref{prop: acf-equivalent}. First, we prove the condition \eqref{a19} for $\mathcal{B}$. To do so, pick an arbitrary $i_0\in [n]$. Given $r,s \in [n]_{-i_0}$, let $W' = \sigma_{\pi_{r,s}}(W)$ and $X' = \lambda_{\pi_{r,s}}(X)$ and note that $\pi_{r,s} \in \nabla_{i_0}$. Therefore, due to Lemma~\ref{lem: invariant-object}, $\Delta_{i_0}(W,X)=\Delta_{i_0}(W',X')$ and then from Lemma~\ref{lem: gi-Deltai}, $\beta_{i_0}(W,X)=\beta_{i_0}(W',X')$. 
	
	As a next step, we prove condition \eqref{a29}. Due to \eqref{ab2-app}, we have
	\begin{align}\label{ab1}
	\beta_{i_0}\left(W,X\right) = \left(x_{i_0},\sum_{r\in [n]_{-i_0}}\psi_2\left(x_r,w_{i_0,r},\sum_{\ell\in [n]_{-r}}\psi_1\left(x_{\ell}, w_{r,\ell}\right)\right)\right).
	\end{align}
	Now replace $W$ by $\sigma_{\pi_{j,i_0}}(W)$ and $X$ by $\lambda_{\pi_{j,i_0}}(X)$ in \eqref{ab1} to get
	\begin{align}
	\beta_{i_0}\left(\sigma_{\pi_{j,i_0}}(W),\lambda_{\pi_{j,i_0}}(X)\right) = \left(x_j,\sum_{r\in [n]_{-j}}\psi_2\left(x_r,w_{j,r},\sum_{\ell\in [n]_{-r}}\psi_1\left(x_{\ell}, w_{r,\ell}\right)\right)\right)=\beta_{j}\left(W,X\right).
	\end{align}
	Hence, $\mathcal{B}$ satisfies condition~\eqref{a29} as well as condition~\eqref{a19}, leading to $\mathcal{B}\in \F_{n,d}$.
\end{proof}

\section{Proof of Theorem~\ref{theo: rho}}\label{app: theo-rho}
We first start with a necessary condition for permutation-compatible functions that will be used throughout the proof.
\begin{corollary}\label{corr: node-func}
	Consider a function $F = (f_1,\ldots,f_n)$ over $\G_{n,d}$ and assume $F\in\mathcal{F}_{n,d}$. Then, for any $G=([n],W,X)\in \G_{n,d}$, we have:
	\begin{align}
			\label{a40p}\text{For every }i\in [n]:&\quad
			f_{i}(\sigma_{\pi}(W),\lambda_{\pi}(X))=f_{i}(W,X)\quad \forall \pi \in \nabla_i.\\
			\label{a41} \text{For every $i,j\in [n]$ with } i\neq j:& \quad f_{j}(W,X) = f_{i}(\sigma_{\pi_{i,j}}(W),\lambda_{\pi_{i,j}}(X)).
		\end{align}
\end{corollary}
\begin{proof}[{\bf Proof of \cref{corr: node-func}}]
    Plug in $\pi\in \nabla_i$ and $\pi=\pi_{i,j}$ in the definition of permutation compatibility in \eqref{a23} to obtain \eqref{a40p} and \eqref{a41}, respectively. 
\end{proof}
\begin{proof}[{\bf Proof of \cref{theo: rho}}]
    Pick an arbitrary $i_0\in [n]$. Due to Corollary~\ref{corr: node-func}, we know that
    \begin{align}\label{a40}
	f_{i_0}(\sigma_{\pi}(W),\lambda_{\pi}(X))=f_{i_0}(W,X)\quad \forall \pi \in \nabla_{i_0}.
	\end{align}
    We claim that there exists a function $\rho$ such that for all  $G=([n],W,X)\in\tilde{\G}_{n,d}$, we have $f_{i_0}(W,X) = \rho\left(\beta_{i_0}(W,X)\right)$.
    
	To see this, consider graphs $G=([n],W,X)$ and $G'=([n],W',X')$ in $\tilde{\G}_{n,d}$. It suffices to show that $\beta_{i_0}(W',X')=\beta_{i_0}(W,X)$ results in $f_{i_0}(W',X')=f_{i_0}(W,X)$.
    If $\beta_{i_0}(W',X')=\beta_{i_0}(W,X)$, \cref{prop: beta} implies that there exists $\pi\in\nabla_{i_0}$ such that $W' = \sigma_{\pi}(W)$ and $X' = \lambda_{\pi}(X)$. Note that we are allowed to use \cref{prop: beta} since the graphs in $\tilde{\G}_{n,d}$ have distinct node features. Hence, from \eqref{a40}, we can write 
    \begin{align}
    f_{i_0}\left(W',X'\right) = f_{i_0}\left(\sigma_{\pi}(W),\lambda_{\pi}(X)\right) = f_{i_0}\left(W,X\right).
    \end{align} 
    Therefore, we have shown the existence of a function $\rho$ such that
    \begin{equation}\label{a24b}
		f_{i_0}(W,X) = \rho\left(\beta_{i_0}(W,X)\right).
	\end{equation} 
	Now, it suffices to prove that $f_{j}(W,X) = \rho(\beta_{j}(W,X))$ holds for $j\in [n]_{-i_0}$. To this end, note that due to Proposition~\ref{prop: beta}, $\mathcal{B}\in \F_{n,d}$ and we also know that $F\in \F_{n,d}$.  Therefore, due to \eqref{a41}, by setting $i=i_0$, for all $j\neq i_0$ we have
    \begin{align}
        \label{z2}f_{j}(W,X) &= f_{i_0}(\sigma_{\pi_{j,i_0}}(W),\lambda_{\pi_{j,i_0}}(X)), \\ \label{z3k}\beta_{j}(W,X) &= \beta_{i_0}(\sigma_{\pi_{j,i_0}}(W),\lambda_{\pi_{j,i_0}}(X)).
    \end{align}
    Hence, putting together \eqref{z2}, \eqref{a24b}, and \eqref{z3k}, respectively, implies the following equation for all $j\neq i_0$:
    \begin{align}\label{a24c}
        f_{j}(W,X) = f_{i_0}(\sigma_{\pi_{j,i_0}}(W),\lambda_{\pi_{j,i_0}}(X)) = \rho\left(\beta_{i_0}(\sigma_{\pi_{j,i_0}}(W),\lambda_{\pi_{j,i_0}}(X))\right) = \rho\left(\beta_{j}(W,X)\right).
    \end{align}
    Therefore, we proved the existence of a function $\rho$ such that $f_{i}(W,X) = \rho\left(\beta_{i}(W,X)\right)$ for all $i\in [n]$ and $G=([n],W,X)\in\tilde{\G}_{n,d}$.
\end{proof}

\section{Proof of Corollary~\ref{corr: cont}}\label{app: corr-cont}
\begin{proof}
Consider the proof of Theorem~\ref{theo: distinct}. Due to the choice of $\phi_1$, $\phi_2$, and $\phi_3$ in \eqref{p1}, \eqref{p2}, and \eqref{p3}, the continuity of $\phi_1$, $\phi_2$, and $\phi_3$ follows from continuity of $\psi_1$, $\psi_2$, and $\rho$. The MEFs introduced in Proposition~\ref{prop: cef} are continuous and thus $\psi_1$ and $\psi_2$ can be chosen from the continuous class of functions.  Moreover, when $\psi_1$ and $\psi_2$ are continuous, $\beta_i$ is continuous. Having the continuity of $f_i$ and $\beta_i$ implies that the function $\rho$ in \eqref{aaa5} also must be continuous over the range of $\beta_i$. Therefore, there exist continuous functions $\phi_1$, $\phi_2$, and $\phi_3$ to generate $F$.   
\end{proof}

\section{Proof of \cref{theo: arbitrary}}\label{app: theo-arbitrary}
\begin{proof}
	It suffices to construct a GNN with the conditions mentioned in the statement. Before introducing this construction, we first show an intermediate result. Using the notation introduced in the statement, fix a basis function $\mathcal{B}=(\beta_1,\ldots,\beta_n)$  defined in \cref{def: beta} over $\mathcal{G}_{n,d+d_0}$. Then we claim that there exists a function $\rho$ such that
	\begin{align}\label{rhoi}
		\rho\left(\beta_i(W,\tilde{X})\right) = \left(W,X,i\right).
	\end{align} 
	To show \eqref{rhoi}, it suffices to show that if $\beta_i(W,\tilde{X}) = \beta_{i'}(W',\tilde{X}')$, then $(W,X,i)=(W',X',i')$, where $G=([n],W,X)$ and $G'=([n],W',X')$ are two graphs in $\mathcal{G}_{n,d}$. Similar to $\tilde{X}$, the term $\tilde{X}'$ denotes the feature matrix $(\tilde{x}_1^{'\top},\ldots, \tilde{x}_n^{'\top})$, where $\tilde{x}_i' = (x_i',y_i)$ with known $y_i$'s given in the statement of the theorem. Having $\beta_i(W,\tilde{X}) = \beta_{i'}(W',\tilde{X}')$, the equality of the first coordinates leads to $\tilde{x}_i=\tilde{x}_{i'}'$. Recall that $\tilde{x}_i = (x_i,y_i)$ and $\tilde{x}_{i'}' = (x_{i'}',y_{i'})$ and thus $(x_i,y_i)= (x_{i'}',y_{i'})$. This results in $y_i=y_{i'}$. Having  $y_i=y_{i'}$ then implies $i=i'$ because $y_1,\ldots,y_n$ are distinct. From $i=i'$, we conclude that $\beta_i(W,\tilde{X}) = \beta_{i}(W',\tilde{X}')$. Knowing that each of $\tilde{X}$ and $\tilde{X}'$ consist of distinct node features, \cref{prop: beta} implies that there exits $\pi\in\nabla_i$ such that $W'=\sigma_{\pi}(W)$ and $\tilde{X}'=\lambda_{\pi}(\tilde{X})$. Particularly, $\tilde{X}'=\lambda_{\pi}(\tilde{X})$ means 
	\begin{align}\label{s3}
		\left(\tilde{x}_1',\ldots, \tilde{x}_n'\right) = \left(\tilde{x}_{\pi(1)},\ldots, \tilde{x}_{\pi(n)}\right).
	\end{align}
	Since the same $y_i$'s are augmented for both $X$ and $X'$, i.e., $\tilde{x}_i' = (x_i',y_i)$ and $\tilde{x}_i = (x_i,y_i)$, \cref{s3} leads to
	\begin{align}\label{s4}
	\left(y_1,\ldots, y_n\right) = \left(y_{\pi(1)},\ldots, y_{\pi(n)}\right).
	\end{align}
	Since $y_1,\ldots,y_n$ are distinct, $\pi$ must be the identity permutation, i.e., $\pi(i)=i$ for all $i\in [n]$. As a result, $W'=\sigma_{\pi}(W)=W$ and $\tilde{X}'=\lambda_{\pi}(\tilde{X})=\tilde{X}$. The equality of the augmented features $\tilde{X}'=\tilde{X}$ then leads to the equality of the actual features, i.e., $X'=X$. Hence, $W'=W$, $X'=X$, and we already showed that $i=i'$. Therefore, $(W,X,i)=(W',X',i')$, which shows the existence of $\rho$ described in \eqref{rhoi}. 
	
	Having established the existence of $\rho$ in \eqref{rhoi}, as a next step, we seek to construct a GNN that represents the given graph function $F(W,X)$. Let $F(W,X)=(f_1(W,X),\ldots,f_n(W,X))$ and define $\theta(W,X,i)$ as follows:
	\begin{align}\label{theta}
		\theta\left(W,X,i\right) = f_i\left(W,X\right).
	\end{align}
	Next, define a GNN over $\mathcal{G}_{n,d+d_0}$ with three iterations as follows. In the first and second iterations, we reach $\beta_i$ at node $i$ similar to the proof of \cref{theo: distinct}. We repeat the argument here for self-sufficiency. We define $\phi_1$, $\phi_2$, and $\phi_3$ such that the resulted GNN $H^{(3)}$ satisfies $H^{(3)}(W,\tilde{X}) = F(W,X)$ for all $G=([n],W,X)\in\G_{n,d}$. First note that the GNN here receives $\tilde{X}$ as the input feature matrix (which lies in $\R^{(d+d_0)\times n}$). Therefore, for all $i\in[n]$ we have
	\begin{align}
		h_i^{(0)} = \tilde{x}_i. 
	\end{align}
	Define the function $\phi_1$ of the GNN as
	\begin{align}\label{p10}
	\phi_1\left(h_j^{(0)},h_{\ell}^{(0)},w_{j,\ell}\right)=\left(\frac{1}{n-1}h_j^{(0)},\,\,\psi_1\left(h_{\ell}^{(0)},w_{j,\ell}\right)\right).
	\end{align}
	This leads to the following formula for all $j\in [n]$:
	\begin{align}
	\nonumber h_j^{(1)} &= \sum_{\ell\in [n]_{-j}}\phi_1\left(h_j^{(0)},h_{\ell}^{(0)},w_{j,\ell}\right)\\
	&= \sum_{\ell\in [n]_{-j}}\left(\frac{1}{n-1}h_j^{(0)},\,\,\psi_1\left(h_{\ell}^{(0)},w_{j,\ell}\right)\right) = \left(\tilde{x}_j,\,\,\sum_{\ell\in [n]_{-j}}\psi_1\left(\tilde{x}_{\ell},w_{j,\ell}\right)\right).
	\end{align}
	Define the function $\phi_2$ of the GNN as
	\begin{align}\label{p20}
	\phi_2\left(h_i^{(1)},h_j^{(1)},w_{i,j}\right)=\left(\frac{1}{n-1}\left[h_i^{(1)}\right]_{1:d+d_0}, \,\,\psi_2\left(\left[h_j^{(1)}\right]_{1:d+d_0},w_{i,j},\left[h_j^{(1)}\right]_{d+d_0+1:\text{end}}\right)\right).
	\end{align}
	Hence, 
	\begin{align}
	\nonumber h_i^{(2)} &= \sum_{j\in [n]_{-i}}\phi_2\left(h_i^{(1)},h_j^{(1)},w_{i,j}\right) 
	\\
	&= \left(\tilde{x}_i,\,\,\sum_{j\in [n]_{-i}}\psi_2\left(\tilde{x}_j,w_{i,j},\sum_{\ell\in [n]_{-j}}\psi_1\left(\tilde{x}_{\ell}, w_{j,\ell}\right)\right)\right)=\beta_i(W,\tilde{X}).
	\end{align}
	So far, we showed that for all $i\in[n]$
	\begin{align}\label{adj}
		h_i^{(2)}(W,\tilde{X}) = \beta_i(W,\tilde{X}).
	\end{align}
	Finally, define the function $\phi_3$ of the GNN as
	\begin{align}\label{p30}
	\phi_3\left(h_i^{(2)},h_j^{(2)},w_{i,j}\right)=\frac{1}{n-1}\,\theta\left(\rho\left(h_i^{(2)}\right)\right),
	\end{align}
	with $\theta$ defined in \eqref{theta} and $\rho$ defined in \eqref{rhoi}. This results in 
	\begin{align}
	h_i^{(3)}(W,\tilde{X}) &= \sum_{j\in [n]_{-i}}\phi_3\left(h_i^{(2)}(W,\tilde{X}),h_j^{(2)}(W,\tilde{X}),w_{i,j}\right)\\
	&= \sum_{j\in [n]_{-i}}\frac{1}{n-1}\,\theta\left(\rho\left(h_i^{(2)}(W,\tilde{X})\right)\right)\\ \label{adj2}&=\theta\left(\rho\left(\beta_i(W,\tilde{X})\right)\right)=\theta\left(W,X,i\right)=f_i(W,X).
	\end{align}
Therefore, $H^{(3)}(W,\tilde{X}) = F(W,X)$ for all $G=([n],W,X)\in\G_{n,d}$.
\end{proof}

\section{Formal statement and proof of \cref{corr: hybrid}}\label{app: corr-arbitrary}
To formally state \cref{corr: hybrid}, consider the following definition.
\begin{definition}
    For $i_1,\ldots,i_k\in [n]$, we use $\nabla_{i_1,\ldots,i_k}$ to denote the set of all permutations over $[n]$ that fix $i_1,\ldots,i_k$. More formally,
    \begin{align}
        \nabla_{i_1,\ldots,i_k} = \left\{\pi \in S_n \mid  \forall j\in \left\{i_1,\ldots,i_k\right\}:\,\, \pi(j)=j\right\}.
    \end{align}
    Note that in $\nabla_{i_1,\ldots,i_k}$, we omit the dependency to $n$ for simplicity.
\end{definition}
\begin{corollary}[Formal]\label{corr: hybrid-formal}
	Consider the graph function $F=(f_1,\ldots,f_n)$ over $\G_{n,d}$ and assume that there exist $i_1,\ldots,i_k\in[n]$ such that $F$ satisfies \eqref{a23} for all permutations $\pi \in \nabla_{i_1,\ldots,i_k}$.\\
	Fix distinct values $y_{i_1},\ldots,y_{i_k}\in \R^{d_0}$ and for every graph $G=([n],W,X)\in \G_{n,d}$ do the following: (i) Choose $y_j$ for $j\in [n]\setminus \{i_1,\ldots,i_k\}$ such that the set of vectors $y_{i_1},\ldots,y_{i_k}$ is expanded to a set of distinct vectors $y_1,\ldots,y_n$. (ii) Augment $y_i$ to the feature $x_i$ to construct $\tilde{x}_i=(x_i,y_i)\in \R^{d+d_0}$ for all $i\in[n]$ and let $\tilde{X}=(\tilde{x}_1^{\top},\ldots,\tilde{x}_n^{\top})$. \\
	Then there exists a GNN $H^{(k)}$ with a finite $k\geq 0$ over $\G_{n,d+d_0}$ such that $H^{(k)}(W,\tilde{X}) = F(W,X)$ for all $G=([n],W,X)\in\G_{n,d}$.
\end{corollary}
\begin{proof}[{\bf Proof of \cref{corr: hybrid-formal}}]
	It suffices to construct a GNN with the conditions mentioned in the statement. Before introducing this construction, we first show an intermediate result. Using the notation introduced in the statement, fix a basis function $\mathcal{B}=(\beta_1,\ldots,\beta_n)$  defined in \cref{def: beta} over $\mathcal{G}_{n,d+d_0}$. Then we claim that there exists a function $\rho$ such that
	\begin{align}\label{rhoi-p}
		f_i(W,X) = \rho\left(\beta_i(W,\tilde{X})\right).
	\end{align} 
	To show \eqref{rhoi-p}, it suffices to show that if $\beta_i(W,\tilde{X}) = \beta_{i}(W',\tilde{X}')$, then $f_i(W,X)= f_i(W',X')$, where $G=([n],W,X)$ and $G'=([n],W',X')$ are two graphs in $\mathcal{G}_{n,d}$. Similar to $\tilde{X}$, the term $\tilde{X}'$ denotes the feature matrix $(\tilde{x}_1^{'\top},\ldots, \tilde{x}_n^{'\top})$, where $\tilde{x}_i' = (x_i',y_i')$ follows the augmentation scheme described in the statement. Moreover, note that $y_{j}=y_{j}'$ holds for $j\in \{i_1,\ldots,i_k\}$ but not necessarily for other $y_j$'s. Having $\beta_i(W,\tilde{X}) = \beta_{i}(W',\tilde{X}')$ and knowing that each of the feature matrices $X'$ and $\tilde{X}'$ consist of distinct node features, \cref{prop: beta} implies that there exists $\pi\in \nabla_i$ such that $W'=\sigma_{\pi}(W)$ and $\tilde{X}'=\lambda_{\pi}(\tilde{X})$. Particularly, $\tilde{X}'=\lambda_{\pi}(\tilde{X})$ means
	\begin{align}\label{s3-p}
		\left(\tilde{x}_1',\ldots, \tilde{x}_n'\right) = \left(\tilde{x}_{\pi(1)},\ldots, \tilde{x}_{\pi(n)}\right).
	\end{align}
	Having $\tilde{x}_i' = (x_i',y_i')$ and $\tilde{x}_i = (x_i,y_i)$, \cref{s3-p} leads to
	\begin{align}\label{s4-p}
	\left(y_1',\ldots, y_n'\right) = \left(y_{\pi(1)},\ldots, y_{\pi(n)}\right).
	\end{align}
	Since each of the collections $y_1,\ldots,y_n$ and $y_1',\ldots,y_n'$ has distinct elements and $y_{j}=y_j'$ for all $j\in \{i_1,\ldots,i_k\}$, the permutation $\pi$ must satisfy $\pi(j)=j$ for all $j\in \{i_1,\ldots,i_k\}$. For such a $\pi$, the statement mentions that \eqref{a23} holds for $F$, i.e., we have $f_{\pi(i)}(W,X)=f_i(\sigma_{\pi}(W),\lambda_{\pi}(X))=f_i(W',X')$. Moreover, recall that $\pi\in \nabla_i$, i.e., $\pi(i)=i$ and thus $f_{i}(W,X)=f_i(W',X')$. As a result, the existence of the function $\rho$ in \eqref{rhoi-p} is proven. 
	
	Having established the existence of $\rho$ in \eqref{rhoi-p}, we follow the steps of GNN construction in \cref{theo: distinct}, i.e., we define $\phi_1$, $\phi_2$, and $\phi_3$ as introduced in \eqref{p1}, \eqref{p2}, and \eqref{p3}, respectively. Under such a construction, $H^{(3)}(W,\tilde{X})=F(W,X)$ for all $G=([n],W,X)\in \G_{n,d}$.
\end{proof}

\section{\bf Proof of \cref{theo: wl}}
\begin{definition}\label{def: wl-di}
For a graph $G=(n,W,C)\in \G^c_{n,d}$, recall that $w_{i,j}\in\{0,1\}$. Define the degree of node $i$ as $d_i=\sum_{j\in [n]_{-i}}w_{i,j}$ and its neighborhood as $N_i=\{j\in [n]_{-i} \mid w_{i,j}=1\}$. Then consider the following sequence of objects:
	\begin{align}
		A_i^{(1)}(G) = d_i, 
	\end{align}
	and for $k> 1$
	\begin{align}
	A_i^{(k)}(G) = \left(A_i^{(k-1)}(G), \,\#\{A_j^{(k-1)}(G) \mid j\in N_i\}\right).
	\end{align}
	Moreover, let $A^{(0)}(G)=n$ and for $k\geq 1$ let 
	\begin{align}
	    A^{(k)}(G) = \#\{A_i^{(k)}(G) \mid i\in [n]\}.
	\end{align}
\end{definition}
\begin{lemma}\label{lem: wl}
    Suppose $G_1$ and $G_2$ are two graphs in $\G_{n,d}^c$. Then 1-WL distinguishes between $G_1$ and $G_2$ if and only if $A^{(k)}(G_1) \neq A^{(k)}(G_2)$ for some $k\geq 1$.
\end{lemma}
\begin{proof}[\bf Proof of \cref{lem: wl}]
    Let $l_i^{(k)}(G)$ denote the label that 1-WL assigns to node $i$ of the graph $G$ at iteration $k$. First, we argue that $l_i^{(k)}(G)$ is in one-to-one correspondence with $A_i^{(k)}(G)$ for every node $i\in[n]$ and iteration $k\geq 1$, i.e., $l_i^{(k)}(G_1)=l_i^{(k)}(G_2)$ if and only if $A_i^{(k)}(G_1)=A_i^{(k)}(G_2)$. To see this, note that when 1-WL starts with identical labels $l_i^{(0)}(G) = c$ for all $i\in [n]$, it produces
    \begin{align}
        l_i^{(1)}(G) = (c,\#\{\underbrace{c,\ldots,c}_{d_i}\}).
    \end{align}
    Since $c$ is fixed, $l_i^{(1)}(G)$ is in one-to-one correspondence with $A_i^{(1)}(G) =d_i$. At each iteration $k$, the 1-WL's updated label satisfies $l_i^{(k)}(G)=(l_i^{(k-1)}(G),\#\{l_j^{(k-1)}(G) \mid j\in N_i\})$. Using the induction hypothesis on $k-1$, $l_r^{(k-1)}(G)$ is in one-to-one correspondence with $A_r^{(k-1)}(G)$ and thus $l_i^{(k)}(G)=(l_i^{(k-1)}(G),\#\{l_j^{(k-1)}(G) \mid j\in N_i\})$ is in one-to-one correspondence with $A_i^{(k)}(G) = (A_i^{(k-1)}(G), \,\#\{A_j^{(k-1)}(G) \mid j\in N_i\})$. Hence, the induction is proven, i.e., $l_i^{(k)}(G)$ is in one-to-one correspondence to $A_i^{(k)}(G)$ for every node $i$ and iteration $k\geq 1$. 
    
    Note that for $G_1,G_2\in\G_{n,d}^c$, 1-WL distinguishes between $G_1$ and $G_2$ if and only if $\#\{l_i^{(k)}(G_1) \mid i\in [n]\} \neq \#\{l_i^{(k)}(G_2) \mid i\in [n]\}$ for some $k\geq 1$. Having the one-to-one correspondence between $l_i^{(k)}(G)$ and $A_i^{(k)}(G)$, discussed above, this is equivalent to $\#\{A_i^{(k)}(G_1) \mid i\in [n]\} \neq \#\{A^{(k)}(G_2) \mid i\in [n]\}$, i.e., $A^{(k)}(G_1) \neq A^{(k)}(G_2)$. Hence, 1-WL distinguishes between $G_1$ and $G_2$ if and only if $A^{(k)}(G_1) \neq A^{(k)}(G_2)$ for some $k\geq 1$.
\end{proof}
\begin{proof}[\bf Proof of \cref{theo: wl}]
    We use the notation of \cref{def: wl-di} throughout the proof. Moreover, for a graph $G=(n,W,C)\in \G^c_{n,d}$ and GNN $H^{(k)}=(h_1^{(k)},\ldots,h_n^{(k)})$, we use the notation $h_i^{(k)}(G)$ to refer to $h_i^{(k)}(W,C)$. We also use $l_i^{(k)}(G)$ to refer to the label that 1-WL assigns to node $i$ of the graph $G$ at iteration $k$. Finally, the initial (identical) labels in the 1-WL algorithm are set to be the node features $x_i=c$.  Based on these notations, the necessity and sufficiency proofs are as follows.\\
	
	\noindent{\bf Necessity.} Suppose 1-WL test cannot distinguish between $G_1$ and $G_2$, then \cref{lem: wl} implies that $A^{(k)}(G_1) = A^{(k)}(G_2)$ for all $k\geq 1$. We want to show that GNNs cannot distinguish between $G_1$ and $G_2$. Considering a GNN $H^{(k)}=(h_1^{(k)},\ldots,h_n^{(k)})$ with the inner functions $\phi_k$, it suffices to show that $\#\{h_i^{(k)}(G_1)\mid i\in [n]\} = \#\{h_i^{(k)}(G_2)\mid i\in [n]\} $ holds for all $k$. Having the functions $\phi_k(\cdot)$  and the value of $c$ fixed, we use induction to show that for every $k\geq 1$, there exists a function $\lambda_k$ such that $h_i^{(k)}(G)= \lambda_k(A_i^{(k)}(G),A^{(k-1)}(G))$. Note that the GNN starts with $h_i^{(0)}=c$ for all $i\in [n]$ and produces the following in the first iteration:
	\begin{align}
		h_i^{(1)}=\sum_{j\in[n]_{-i}}\phi_1(c,c,w_{i,j}) = d_i\, \phi_1(c,c,1) + (n-1-d_i)\, \phi_1(c,c,0).
	\end{align}
	Having functions $\phi_k$ and $c$ fixed, $h_i^{(1)}$ is only a function of $A_i^{(1)}(G)=d_i$ and $A^{(0)}(G)=n$. Hence, there exists a function $\lambda_1$ such that $h_i^{(1)}(G)= \lambda_1(A_i^{(1)}(G),A^{(0)}(G))$. Given the induction hypothesis for $k$, we assume $h_i^{(k)}(G)= \lambda_k(A_i^{(k)}(G),A^{(k-1)}(G))$ for all $i\in [n]$ and prove it for $k+1$. To this end, note that
	\begin{align}
		h_i^{(k+1)} &=\sum_{j\in[n]_{-i}}\phi_k\left(h_i^{(k)},h_j^{(k)},w_{i,j}\right) \\
		&= \sum_{j\in N_i} \phi_k\left(\lambda_k\left(A_i^{(k)}(G),A^{(k)}(G)\right),\lambda_k\left(A_j^{(k)}(G),A^{(k)}(G)\right),1\right) \\& + \sum_{j\notin N_i\cup \{i\}} \phi_k\left(\lambda_k\left(A_i^{(k)}(G),A^{(k)}(G)\right),\lambda_k\left(A_j^{(k)}(G),A^{(k)}(G)\right),0\right).
	\end{align}
	Hence, $h_i^{(k+1)}$ can be uniquely determined in terms of $A_i^{(k)}(G)$, $A^{(k)}(G)$, $\#\{A_j^{(k)}(G) \mid j\in N_i\}$, and $\#\{A_j^{(k)}(G) \mid j\notin N_i\cup \{i\}\}$. These quantities themselves can be uniquely determined in terms of $A_i^{(k+1)}(G)$ and $A^{(k)}(G)$. To see this, note that $A_i^{(k)}(G)$ and $\#\{A_j^{(k)}(G) \mid j\in N_i\}$ are the first and the second component of $A_i^{(k+1)}(G)$. Further, $\#\{A_j^{(k)}(G) \mid j\notin N_i\cup \{i\}\}$ can be obtained by removing $A_i^{(k)}(G)$ and the elements of $\#\{A_j^{(k)}(G) \mid j\in N_i\}$ from $A^{(k)}(G)$.
	Hence, the function $\lambda_{k+1}$ exists such that $h_i^{(k+1)}(G)= \lambda_{k+1}(A_i^{(k+1)}(G),A^{(k)}(G))$ for all $i\in [n]$, which completes the induction.
	
	Having established the existence of $\lambda_k$ as described above, we proceed as follows: Given $A^{(k)}(G_1) = A^{(k)}(G_2)$ for all $k\geq 1$, we want to show that $\#\{h_i^{(k)}(G_1)\mid i\in [n]\} = \#\{h_i^{(k)}(G_2)\mid i\in [n]\} $ holds for all $k\geq 1$. To see this, note that 
	\begin{align}
	    \#\{h_i^{(k)}(G)\mid i\in [n]\} = \#\{\lambda_{k}(A_i^{(k)}(G),A^{(k-1)}(G))\mid i\in [n]\}.
	\end{align}
	Hence, $\#\{h_i^{(k)}(G)\mid i\in [n]\}$ is uniquely determined in terms of $A^{(k-1)}(G)$ and $\#\{A_i^{(k)} \mid i\in [n]\}=A^{(k)}(G)$. Therefore, the equations $A^{(k)}(G_1) = A^{(k)}(G_2)$ for all $k\geq 1$ leads to $\#\{h_i^{(k)}(G_1)\mid i\in [n]\} = \#\{h_i^{(k)}(G_2)\mid i\in [n]\}$ for all $k\geq 1$. \\
	
	\noindent{\bf Sufficiency.} For the sufficiency part, suppose 1-WL test can distinguish between $G_1$ and $G_2$. Therefore, there exists $k\geq 1$ such that $A^{(k)}(G_1)\neq A^{(k)}(G_2)$. To show that GNNs can also distinguish between $G_1$ and $G_2$, it suffices to build a GNN $H^{(k)}=(h_1^{(k)},\ldots,h_n^{(k)})$ such that for any two graphs $G_1$ and $G_2$, the equality $\#\{h_i^{(k)}(G_1) \mid i\in[n]\}=\#\{h_i^{(k)}(G_2) \mid i\in[n]\}$ implies $A^{(k)}(G_1)= A^{(k)}(G_2)$. To build the aforementioned GNN, we set 
	\begin{align}
	    &\phi_1(h_i^{(0)},h_j^{(0)},w_{i,j}) = w_{i,j},\\
	    \label{y1} \text{for } k > 1:\quad&\phi_k(h_i^{(k-1)},h_j^{(k-1)},w_{i,j}) =\left(\frac{1}{n-1}h_i^{(k-1)} ,  \psi_k\left(w_{i,j}\,h_j^{(k-1)}\right)\right),
	\end{align}
	where the multiplication $w_{i,j}\,h_j^{(k-1)}$ is either $h_j^{(k-1)}$ or zero depending on $w_{i,j}\in \{0,1\}$. Moreover, note that $\psi_k\in \Psi_{m_k,n-1}$ is chosen based on the candidates introduced in \cref{prop: cef}, for some appropriate $m_k$. Having this GNN, as a next step, we show that there exists a function $\lambda_k$ such that  $A_i^{(k)}(G)= \lambda_k(h_i^{(k)}(G))$. We show this by induction on $k$. 
	For $k=1$  
	\begin{align}
	    h_i^{(1)} = \sum_{j\in [n]_{-i}}\phi_1\left(h_i^{(0)},h_j^{(0)},w_{i,j}\right) = \sum_{j\in [n]_{-i}}w_{i,j} = d_i.
	\end{align} 
	Hence, $h_i^{(1)}(G) = A_i^{(1)}(G)=d_i$ and thus $\lambda_1$ exists. Suppose the induction hypothesis holds for $k-1$. We want to prove it for $k$. From \cref{y1}, we have
	\begin{align}
	    \label{y3}h_i^{(k)} = \sum_{j\in [n]_{-i}}\phi_k\left(h_i^{(k-1)},h_j^{(k-1)},w_{i,j}\right) = \left(h_i^{(k-1)} ,  \sum_{j\in [n]_{-i}}\psi_k\left(w_{i,j}\,h_j^{(k-1)}\right)\right).
	\end{align}
	Due to the definition of MEFs (see \cref{def: cef}), having $\sum_{j\in [n]_{-i}}\psi_k\left(w_{i,j}\,h_j^{(k-1)}\right)$ uniquely determines the following multiset
	\begin{align}
	    \label{y2}\#\left\{w_{i,j}h_j^{(k-1)}\mid j\in [n]_{-i}\right\} = \#\left\{h_j^{(k-1)}\mid j\in N_i\right\} \cup \#\left\{\underbrace{\mathbf{0},  \ldots,\mathbf{0}}_{n-1-d_i}\right\},
	\end{align}
	where $\mathbf{0}$ is a vector of all zeros with the same size as $h_j^{(k-1)}$ corresponding to $w_{i,j}\,h_j^{(k-1)}$ when $w_{i,j}=0$. Note that \eqref{y2} also uniquely determines $\#\{h_j^{(k-1)}\mid j\in N_i\}$. This is because $h_j^{(k-1)}$ cannot be $\mathbf{0}$ if node $j$ is not an isolated node and we assumed that there are no isolated nodes in the graph.
	
	Now let us summarize the induction argument: Given $h_i^{(k)}$ in \eqref{y3}, we obtain $h_i^{(k-1)}$ and $\sum_{j\in [n]_{-i}}\psi_k\left(w_{i,j}\,h_j^{(k-1)}\right)$. Then, $h_i^{(k-1)}$ uniquely determines $A_i^{(k-1)}(G)$ through $A_i^{(k-1)}(G)= \lambda_{k-1}(h_i^{(k-1)}(G))$. Moreover, $\sum_{j\in [n]_{-i}}\psi_k\left(w_{i,j}\,h_j^{(k-1)}\right)$ uniquely determines $\#\{h_j^{(k-1)}\mid j\in N_i\}$, as discussed above. The multiset $\#\{h_j^{(k-1)}\mid j\in N_i\}$ then uniquely determines $\#\{A_j^{(k-1)}\mid j\in N_i\}$  because $A_j^{(k-1)}(G)= \lambda_{k-1}(h_j^{(k-1)}(G))$. Therefore, given $h_i^{(k)}$, we can uniquely obtain $A_i^{(k-1)}(G)$ and $\#\{A_j^{(k-1)}\mid j\in N_i\}$, which are the components of $A_i^{(k)}(G)$. Hence,  $A_i^{(k)}(G)$ is a function $h_i^{(k)}(G)$ and thus $\lambda_k$ exists.
	
	Having established the existence of $\lambda_k$ described above, implies that $A^{(k)}(G)= \#\{A_i^{(k)}(G) \mid i\in [n]\}$ is a function of $\#\{h_i^{(k)}(G)\mid i\in [n]\}$. As a result, the equality $\#\{h_i^{(k)}(G_1) \mid i\in[n]\}=\#\{h_i^{(k)}(G_2) \mid i\in[n]\}$ implies $A^{(k)}(G_1)= A^{(k)}(G_2)$ which concludes the proof.
\end{proof}

\section{Proof of Lemma~\ref{lem: SP}}
\begin{proof}
	We want to show that if $\pi\in \nabla_1$, then $f_{\pi(i)}(W) = f_{i}(\sigma_{\pi}(W))$ holds for every $i\in [n]$. For $i=1$, we need to show that $f_{\pi(1)}(W) = f_{1}(\sigma_{\pi}(W))$ or equivalently $f_{1}(W) = f_{1}(\sigma_{\pi}(W))$ because $\pi(1)=1$. This holds since $f_1(W)=0$ for every valid weight matrix $W$. Therefore, it suffices to prove the argument when $i\neq 1$.
	
	As the next step, we show that if $r,s\in [n]_{-1}$ with $r\neq s$, then $f_{\pi(i)}(W) = f_{i}(\sigma_{\pi}(W))$ holds for $\pi=\pi_{r,s}$. The case $i=1$ is proven earlier. For the case $i=s$, we have $\pi_{r,s}(s)=r$ which means that we need to show that $f_{r}(W) = f_{s}(\sigma_{\pi_{r,s}}(W))$. This can be either verified algebraically using \eqref{Dist} or by the following combinatorial argument: \\
	Note that in general the distance between nodes $i$ and $j$ in the graph with the weight matrix $W$ equals the distance between nodes $\pi(i)$ and $\pi(j)$ in the graph with the weight matrix $\sigma_{\pi}(W)$. Now, consider $\pi=\pi_{r,s}$. This argument implies that the distance between $r$ and $1$ in the graph with the weight matrix $W$, i.e., $f_{r}(W)$ is equal to the distance between $\pi_{r,s}(r)=s$ and $\pi_{r,s}(1)=1$ in the graph with the weight matrix $\sigma_{\pi_{r,s}}(W)$, i.e., $f_{s}(\sigma_{\pi_{r,s}}(W))$. Therefore, $f_{r}(W)=f_{s}(\sigma_{\pi_{r,s}}(W))$. Using this argument for any $i\in [n]$, we conclude that for every distinct pair $r,s\in [n]_{-1}$, 
	\begin{align}\label{ary}
		f_{\pi_{r,s}(i)}(W) = f_{i}(\sigma_{\pi_{r,s}}(W)),
	\end{align}
	holds for all $i\in [n]$ and all valid weight matrices $W$.
		As the final step, note that due to \cref{prop: node-aif-basis}, \cref{ary} results in $f_{\pi(i)}(W) = f_{i}(\sigma_{\pi}(W))$ for every $\pi\in \nabla_1$.
\end{proof}

\section{Formal statement and proof of \cref{prop: SP}}\label{app: prop-SP}
\cref{prop: SP} is formally stated as follows.
\begin{proposition}[Formal]\label{prop: Sp-formal}
	Let $F(W,X)$ be the distance-to-node-1 function defined in \cref{node1} of \cref{ex: funcs}. Then the following holds: 
	\begin{enumerate}[(i)]
	    \item $F \notin \mathcal{F}_{n,d}$.
	    \item Since $F(W,X)$ ignores $X$, we use $F(W)$ and assume that some $X_0\in \R^{d\times n}$ is given (where $X_0=(y_1^{\top},\ldots,y_n^{\top})$). We also let $\G_{n,d}^0 = \{([n],W,X)\in\G_{n,d} \mid X=X_0\}$. Then, there exists a GNN $H^{(k)}$ with some finite $k\geq 0$ such that $H^{(k)}(W,X_0)=F(W)$ for all $G=([n],W,X_0)\in\G_{n,d}^{0}$ if and only if $y_1\neq y_{s}$ for all $s \neq 1$. One simple example for such $X_0$ is $X_0=(1,0,\ldots,0)$ with $d=1$.
	\end{enumerate}
\end{proposition}
\begin{proof}[{\bf Proof of \cref{prop: Sp-formal}}]
    First we prove an intermediate result. We claim that if $y_1=y_s$ for some $s\neq 1$, then no permutation-compatible graph function $Q$ over $\G_{n,d}$ can exist such that $Q(W,X_0) = F(W)$ for all $G=([n],W,X_0)\in\G_{n,d}^{0}$. In particular, this will lead to the proof of part (i) and leads to part (ii) as we will discuss below. To show the claim, we use proof by contradiction. Suppose such a function $Q$ exists. Let $Q=(q_1,\ldots,q_n)$ and note that $Q$ satisfies \eqref{a23}. Letting $\pi=\pi_{1,s}$ and $X=X_0$ in \eqref{a23}, we have $q_{s}(W,X_0) = q_1(\sigma_{\pi_{1,s}}(W),\lambda_{\pi_{1,s}}(X_0))$ for all $G=([n],W,X_0)\in\G_{n,d}$. Since $y_1=y_s$, we have $\lambda_{\pi_{1,s}}(X_0)=X_0$, which implies the following for all $G=([n],W,X_0)\in\G_{n,d}^{0}$:
		\begin{align}\label{ju}
			q_{s}(W,X_0) = q_1(\sigma_{\pi_{1,s}}(W),\lambda_{\pi_{1,s}}(X_0)) = q_1(\sigma_{\pi_{1,s}}(W),X_0).
		\end{align}
		 Note that $q_1(W,X_0)=f_1(W)=0$ and $q_s(W,X_0)=f_s(W)$ is the distance between node $s$ and node $1$. Since $q_1(W,X_0)=f_1(W)=0$ holds for all valid weight matrices $W$, we also have $q_1(\sigma_{\pi_{1,s}}(W),X_0)=0$. This together with \eqref{ju} implies that $q_{s}(W,X_0)=0$ and consequently $f_s(W)=0$ for all valid weight matrices $W$. This means that the minimum distance between node $s\neq 1$ and node $1$ is zero for all weight matrices $W$ which is obviously a contradiction. Due to this contradiction, the claim is proven. Based on this argument, we prove parts (i) and (ii) as follows.
	\begin{enumerate}[(i)]
		\item  If $F \in \mathcal{F}_{n,d}$, then $F(W,X)$ must satisfy \eqref{a23} for all valid feature matrix $X$. In particular, consider $X=X_0$, where $y_1=y_s$ for some $s\neq 1$. Based on the argument above, $F(W,X_0)$ cannot satisfy \eqref{a23}. Hence, $F \notin \mathcal{F}_{n,d}$.
		
		\item {\bf Necessity}. Suppose there exists $s \in [n]_{-1}$ such that $y_1=y_{s}$. We use proof by contradiction. Suppose there exists a GNN $H^{(k)}=(h_1^{(k)},\ldots,h_n^{(k)})$ with some finite $k\geq 0$ such that $H^{(k)}(W,X_0)=F(W)$ for all $G=([n],W,X_0)\in\G_{n,d}^{0}$. Due to \cref{theo: necessity}, $H^{(k)}$ is permutation compatible. This contradicts the claim shown in the beginning of the proof. Hence, such a GNN does not exists.
		
		{\bf Sufficiency}. Suppose $y_1\neq y_s$ for all $s\neq 1$. Note that $X_0=(y_1^{\top},\ldots,y_s^{\top})$ is a universal feature that we use for all graphs. This means we can determine if $i=1$ given $y_i$. We know that the Bellman-Ford dynamic program stated below computes the distance to node $1$ given the identification of node $1$. We first formally state the Bellman-Ford algorithm over a fully connected weighted graph and then show how it can be represented by a GNN of the form expressed in \cref{def: gnn}. To this end, let $M=\max_{i,j\in [n]}{|w_{i,j}|}$ and define the output of the algorithm for node $i$ in iteration $k$ as $l_i^{(k)}$. Set $l_i^{(0)} = 2M$ for all $i\neq 1$ and $l_1^{(0)} = 0$. Perform the following procedure for $k\geq 1$: $l_1^{(k)}=0$ and 
		\begin{align}
		    \label{BF}\text{if } i \neq 1: \quad &l_i^{(k)} = \min\left(l_i^{(k-1)},\,\,\min_{j\in [n]_{-i}} l_j^{(k-1)}+w_{i,j}\right).
		\end{align}
		It is straightforward to see that \eqref{BF} eventually assigns to each node its distance to node $1$ for sufficiently large $k$. Next, we show that there exists a GNN of the form expressed in \cref{def: gnn} that generate $l_i^{(k)}$ using $X_0=(y_1,\ldots,y_s)$ as its initialisation. Due to \cref{prop: extended gnn}, it suffices to construct an Extended-GNN $E^{(k)}=(e_1^{(k)},\ldots,e_n^{(k)})$, defined in \cref{def: gnn2} that generates the shortest path based on \cref{BF}. Starting with $e_i^{(0)}=y_i$, we define
		\begin{align}
		    \nonumber e_i^{(1)} &=\Phi_1\left(e_i^{(0)},\#\left\{(e_j^{(0)},w_{i,j})\mid j\in[n]_{-i}\right\}\right) = \left\{\begin{array}{ll}
		         1,&\text{if }  e_i^{(0)}= y_1,\\
		         0,&\text{otherwise},
		    \end{array}\right.\\
		    \nonumber e_i^{(2)} &= \Phi_2\left(e_i^{(1)},\#\left\{(e_j^{(1)},w_{i,j})\mid j\in[n]_{-i}\right\}\right) = \left(e_i^{(1)},\max_{j\in[n]_{-i}} |w_{i,j}|\right),\\
		     \nonumber e_i^{(3)} &= \Phi_3\left(e_i^{(2)},\#\left\{(e_j^{(2)},w_{i,j})\mid j\in[n]_{-i}\right\}\right)= \left(\left[e_i^{(2)}\right]_1,2\max_{j\in[n]_{-i}} \left[e_j^{(2)}\right]_2\right),\\
		     \nonumber e_i^{(4)} &= \Phi_3\left(e_i^{(3)},\#\left\{(e_j^{(3)},w_{i,j})\mid j\in[n]_{-i}\right\}\right)= \left(\left[e_i^{(3)}\right]_1, \left[e_i^{(3)}\right]_2\left(1-\left[e_i^{(3)}\right]_1\right)\right).
		\end{align}
		It is straightforward to see that $e_1^{(4)} = (1,0)$ and $e_i^{(4)} = (0,2M)$ for $i\neq 1$. For the iterations $k\geq 5$, we set the first coordinate of $e_i^{(k)}$ to remain as the indicator of node $1$ and the second coordinate to compute the step \eqref{BF}. To this end, we set
		\begin{align}
		    e_i^{(k)} &=\Phi_k\left(e_i^{(k-1)},\#\left\{(e_j^{(k-1)},w_{i,j})\mid j\in[n]_{-i}\right\}\right) = \left(\left[e_i^{(k-1)}\right]_1,\,\Lambda_k\right),
		\end{align}
		where the function $\Lambda_k$ implements the update \eqref{BF} as follows:
		\begin{align}
		   \label{y5}\Lambda_k = \left\{\begin{array}{ll}
		         0,&\,\,\text{if }  \left[e_i^{(k-1)}\right]_1= 1,\\
		         \min\left(\left[e_i^{(k-1)}\right]_2,\,\,\min_{j\in [n]_{-i}} \left[e_j^{(k-1)}\right]_2+w_{i,j}\right),&\,\,\text{otherwise}.
		    \end{array}\right.
		\end{align}
		It is straightforward to see that $\left[e_i^{(k+4)}\right]_2=l_i^{(k)}$ for all $k\geq 0$ and all $i\in [n]$. Note that there exists $K$ such that $l_i^{(K)}$ is equal to the distance between node $i$ and node $1$, i.e., $l_i^{(K)}=f_i$. For $k\leq K+4$, we continue as \eqref{y5} and for $k=K+5$, we set 
		\begin{align}
		    \nonumber e_i^{(K+5)} =\Phi_{K+5}\left(e_i^{(K+4)},\#\left\{(e_j^{(K+4)},w_{i,j})\mid j\in[n]_{-i}\right\}\right) = \left[e_i^{(K+4)}\right]_2 = l_i^{(K)} = f_i.
		\end{align}
		Hence, $E^{(K+5)}(W,X_0)=F(W)$ for all $G=([n],W,X_0)\in\G_{n,d}^{0}$. Note that $E^{(k)}$ is an Extended-GNN. Therefore, \cref{prop: extended gnn} implies that there exists a GNN ${H}^{(L)}$ for some finite $L$ such that ${H}^{(L)}(W,X)=E^{(K+5)}(W,X)$ for all $G=([n],W,X)\in\G_{n,d}$. As a result, ${H}^{(L)}(W,X_0)=F(W)$ for all $G=([n],W,X_0)\in\G_{n,d}^{0}$. In fact, the proof of \cref{prop: extended gnn} shows that ${H}^{(L)}$ exists with $L=2K+10$.
\end{enumerate}
\end{proof}

\end{document}